\documentclass[letterpaper]{article}
\usepackage[preprint]{aaai2027}
\usepackage[hyphens]{url}
\usepackage{graphicx}
\usepackage{natbib}
\usepackage{bibunits}
\usepackage{caption}
\usepackage{booktabs}
\usepackage{tabularx}
\usepackage{amsmath}
\usepackage{amssymb}
\usepackage{amsthm}
\usepackage{nicefrac}
\usepackage{algorithm}
\usepackage[noend]{algpseudocode}

\newtheorem{theorem}{Theorem}
\newtheorem{proposition}{Proposition}
\newtheorem{lemma}{Lemma}
\newtheorem{corollary}{Corollary}
\theoremstyle{remark}
\newtheorem{remark}{Remark}

\providecommand{\Description}[1]{}
\providecommand{\texorpdfstring}[2]{#1}

\newcommand{\nrm}{\mathrm{norm}}
\newcommand{\lamn}{\lambda_{\nrm}}
\newcommand{\bdev}{\mathrm{burn\_dev}}
\DeclareRobustCommand{\padct}{PA\nobreakdash-DCT}
\newcommand{\rt}{r_t}
\newcommand{\ut}{u_t}
\newcommand{\ctilde}{\tilde c_t}
\newcommand{\thetaq}[2]{\theta_q^{#1,#2}}
\newcommand{\thetac}[2]{\theta_c^{#1,#2}}
\newcommand{\Sigmaq}[2]{\Sigma_q^{#1,#2}}
\newcommand{\Sigmac}[2]{\Sigma_c^{#1,#2}}
\newcommand{\hatu}[2]{\hat u^{#1,#2}}
\newcommand{\hatc}[2]{\hat c^{#1,#2}}
\newcommand{\Reg}{\mathrm{PA\text{-}Reg}}
\newcommand{\paReg}{\mathrm{PA\text{-}Reg}}

\title{402Pilot: An x402 Decision Layer for Autonomous Agent Micropayments}

\author{
Yin Li\textsuperscript{\rm 1},
Yanbo He\textsuperscript{\rm 1},
Boo-Ho Yang\textsuperscript{\rm 2},
Rav Lawana\textsuperscript{\rm 3},\\
Ziyue Li\textsuperscript{\rm 4},
Wei Zeng\textsuperscript{\rm 1},
Jing Tang\textsuperscript{\rm 1},
Fugee Tsung\textsuperscript{\rm 1,\rm 5}\corresponding
}

\affiliations{
\textsuperscript{\rm 1}The Hong Kong University of Science and Technology (Guangzhou), Guangzhou, China\\
\textsuperscript{\rm 2}MOVENSYS Inc., Seongnam-si, Republic of Korea\\
\textsuperscript{\rm 3}Schneider Electric, Shanghai, China\\
\textsuperscript{\rm 4}Technical University of Munich, Munich, Germany\\
\textsuperscript{\rm 5}The Hong Kong University of Science and Technology, Hong Kong, China\\
yligt@connect.hkust-gz.edu.cn,
yhe376@connect.hkust-gz.edu.cn,
byang@movensys.com,\\
rav.lawana@se.com,
ziyue.li@tum.de,
weizeng@hkust-gz.edu.cn,
jingtang@hkust-gz.edu.cn,\\
season@ust.hk
}

\begin{document}

\maketitle

\begin{abstract}
Programmable-payment protocols such as x402 enable per-request micropayments,
but they do not determine which payable service an autonomous agent should buy
under a finite wallet. We formulate this buyer-side problem as
\emph{agent-native payment decision-making}: contextual provider selection under
wallet pressure, chosen-only paid feedback, and changing market conditions. We
propose 402Pilot, a protocol-agnostic buyer-side decision layer between
autonomous agents and payment execution that implements purchasing policies for
selecting among payable providers. We instantiate it with PA-DCT, a
payment-aware discounted contextual Thompson-sampling policy that adapts
purchasing decisions under wallet pressure while learning from post-payment
feedback. To evaluate buyer-side payment policies, we introduce 402Pilot-Bench,
a frozen-replay benchmark spanning 823 tasks, five heterogeneous provider
pipelines, and three market regimes, each evaluated over 30 paired seeds. PA-DCT
achieves the strongest fixed-wallet adaptive trade-off among non-oracle
policies: it maintains competitive service quality while spending only
39--43\% of the wallet and reallocates spending as market conditions change. It
attains the best non-oracle PA-gap/$T$ under the price shock and the best mean
and worst-case ranks across the nine scenario--metric combinations of quality,
ROI, and PA-gap/$T$. Comparisons with learning baselines and component ablations
further support the effectiveness and design of the proposed decision policy.
These results suggest that programmable payment must be complemented by
buyer-side decision-making capable of learning service value and adapting
purchasing decisions accordingly.

\end{abstract}

\begin{links}
  \link{Code and data}{https://github.com/MCCodeAI/402Pilot}
\end{links}

\begin{bibunit}[aaai2027]
\section{Introduction}
\label{sec:intro}

The rapid development of x402 brings protocol capabilities that traditional payment flows do not support: machine-readable payment requirements, per-request micropayments, and low-friction settlement. With further advances such as x402 V2~\citep{coinbase2025x402v2}, additional capabilities---service discovery, dynamic service metadata, and pricing---are exposed to agents through programmatic interfaces, making buyer-side service selection an immediate practical problem.

Several lines of prior work supply pieces of this problem, but not the full buyer-side setting. Budgeted, contextual, and non-stationary bandits offer tools for learning under resource constraints, task-dependent conditions, and changing environments~\citep{badanidiyuru2018bandits,agrawal2013contextualts,qi2023dsts}. Cost-aware LLM routing similarly considers quality--cost trade-offs, but typically does not model repeated purchasing decisions under a persistent wallet and post-payment feedback~\citep{chen2023frugalgpt,ong2024routellm}. These directions do not jointly capture a protocol-mediated paid-feedback loop in which the buyer selects and pays for one service, then learns only from the outcome and cost of that selection. Payment protocols address execution, while auction mechanisms study strategic provider bidding~\citep{patra2026reverseauction}; neither provides an online spending policy for a single autonomous buyer. Existing work therefore does not jointly address the buyer-side requirements of agent-operated micropayments.

We refer to this buyer-side setting as agent-native payment decision-making: an autonomous agent repeatedly makes irreversible purchasing decisions among payable providers under task context and wallet pressure, learns from the performance and cost of each selected service, and adapts to market changes. Unlike human-operated selection, this process must run continuously without human resetting, retuning, or comparison of unchosen outcomes. Micropayments also make repeated low-cost purchases possible, allowing the agent to accumulate feedback over many fine-grained decisions. These properties make autonomous micropayment an online economic decision problem, not only a settlement problem.

We introduce 402Pilot as a protocol-agnostic buyer-side decision layer between the agent and the payment stack. For each request, it derives the task context and obtains the affordable provider set, then applies a buyer policy to select a provider. 402Pilot forwards the selection for payment execution and converts the resulting service outcome and receipt into learning feedback.

We instantiate this policy as PA-DCT (Payment-Aware Discounted Contextual Thompson Sampling), a contextual bandit that maintains discounted posteriors over service utility and realized cost. It uses Thompson sampling for provider selection under wallet pressure and discounts older observations to adapt as market conditions change.

We evaluate buyer-side policies on \emph{402Pilot-Bench}, a reproducible frozen-replay benchmark with $823$ tasks, five heterogeneous provider pipelines, and stationary, reliability-shock, and price-shock scenarios, each evaluated over $30$ paired seeds. Among the non-oracle policies, PA-DCT delivers the strongest fixed-wallet adaptive trade-off: it maintains competitive quality while using only $39$--$43\%$ of the available wallet across all three scenarios and reallocates spending appropriately as reliability and prices change. Against both non-learning policies and learning baselines, it attains the best mean rank ($3.3/10$) and best worst-case rank ($6/10$) across the nine scenario--metric comparisons of quality, ROI, and PA-gap/$T$. Ablations further clarify how PA-DCT's components contribute to solvency, adaptation, and robustness.

We make three contributions.

\textbf{C1.} We formulate \emph{agent-native payment decision-making}, a new buyer-side online decision problem for autonomous micropayment agents, and introduce \emph{402Pilot} as the missing decision layer above payment execution (\S\ref{sec:problem}--\S\ref{sec:layer}).

\textbf{C2.} We instantiate 402Pilot with PA-DCT, a payment-aware discounted contextual Thompson-sampling policy that jointly handles wallet pressure, task context, and reliability and price drift, yielding adaptive provider selection under market change (\S\ref{sec:method}).

\textbf{C3.} We introduce \emph{402Pilot-Bench} for reproducible evaluation of buyer-side payment policies used by autonomous agents. Its frozen-replay design supports comparisons under paid feedback, finite budgets, and market shocks (\S\ref{sec:eval}).

\section{Related Work}
\label{sec:related-work}

\paragraph{Programmable payment protocols and discovery.}
Programmable payment protocols expose the execution substrate for agent payments. x402 embeds per-request payment into HTTP, and its V2 release adds discovery and service metadata for payable APIs~\citep{coinbase2025x402v2}. AP2 provides a payment-agnostic authorization framework for agent-led commerce~\citep{google2025ap2}; A2A--x402 brings x402-style cryptocurrency payments into the Agent-to-Agent protocol~\citep{a2ax402}; and A402 binds cryptocurrency payment to service execution through atomic service channels~\citep{a402}. A recent SoK organizes blockchain-based agent payments around discovery, authorization, execution, and accounting~\citep{sok2026agentpayments}. These systems make payable services discoverable and executable, but leave the buyer's spending policy outside the protocol. 402Pilot takes that policy as its object of study.

\paragraph{Budgeted, contextual, and non-stationary bandits.}
Classical bandit methods provide relevant algorithmic tools for this setting. Budgeted Thompson Sampling samples reward and cost posteriors under a finite budget~\citep{xia2015bts}, while bandits with knapsacks give a broader constrained-resource formulation~\citep{badanidiyuru2018bandits}. Contextual bandits estimate action value from task features, with Contextual Thompson Sampling and LinUCB as canonical examples~\citep{agrawal2013contextualts,li2010linucb}. Contextual bandits with knapsacks combine context with resource constraints~\citep{agrawal2016linearcbwk}. For non-stationary environments, discounted and sliding-window methods discount or window old observations~\citep{garivier2011discounted,trovo2020sliding,qi2023dsts}, and weighted linear bandits extend this idea to drifting contextual rewards~\citep{russac2019weighted}. However, none of these methods combines budget awareness, task context, and drift adaptation within a protocol-mediated paid-feedback loop.

\paragraph{Cost-aware LLM routing and provider selection.}
FrugalGPT uses cascades to reduce inference cost while preserving quality, while RouteLLM, Hybrid LLM, AutoMix, MixLLM, and LLM Bandit learn prompt-to-model assignments from quality, preference, cost, or latency signals~\citep{chen2023frugalgpt,ong2024routellm,ding2024hybridllm,madaan2024automix,mixllm2025,llmbandit2025}. These methods typically operate within a fixed serving stack with known candidates and listed prices. 402Pilot instead studies wallet-constrained endpoint selection with chosen-only feedback on realized outcomes and costs. Reverse-auction work considers strategic provider bidding~\citep{patra2026reverseauction}, whereas 402Pilot assumes posted prices and focuses on a single buyer's online spending policy.

\section{Buyer-Side Decision Problem}
\label{sec:problem}

We formalize buyer-side payment decision-making as a sequential
provider-selection problem in which an autonomous buyer must allocate a finite
wallet across competing services.

\subsection{Decision Setting}
\label{sec:decision-setting}

A single autonomous buyer interacts with $K$ payable providers over a horizon
of $T$ rounds, starting with wallet balance $B_0=B$. At round
$t\in\{0,\ldots,T-1\}$, a task arrives with context
$\mathbf{x}_t\in\mathcal X$, while $B_t$ denotes the buyer's current wallet
balance.

Each provider $a\in\{1,\ldots,K\}$ carries a charge $p_{a,t}$. The buyer may
select only from the affordable set
\begin{equation}
\mathcal A_t
=
\left\{
a\in\{1,\ldots,K\}
:
p_{a,t}\le B_t
\right\}.
\label{eq:affordable-set}
\end{equation}
If $\mathcal A_t=\emptyset$, the interaction terminates. Let $\tau\le T$ denote
the resulting number of completed purchases.

Let $H_t$ denote the information available before selection, including the
current task context, wallet state, affordable set, any provider metadata
exposed before payment, and feedback from previous purchases. An admissible
policy selects only affordable providers:
\begin{equation}
a_t
\sim
\pi_t(\cdot\mid H_t),
\qquad
\operatorname{supp}
\left(
\pi_t(\cdot\mid H_t)
\right)
\subseteq
\mathcal A_t.
\label{eq:admissible-policy}
\end{equation}
Let $\Pi_{\mathrm{adm}}$ denote the class of admissible policies.

The selected provider produces
\begin{equation}
(q_t,f_t)
\sim
\mathcal D_{a_t,t}
\left(
\cdot
\mid
\mathbf{x}_t
\right),
\label{eq:provider-outcome}
\end{equation}
where $q_t\in[0,1]$ is the realized task quality and
$f_t\in\{0,1\}$ indicates paid non-delivery: the service is charged but
returns no usable output. After payment, the buyer observes $(q_t,f_t)$ and
the realized charge $c_t=p_{a_t,t}$, and the wallet evolves as
\begin{equation}
B_{t+1}=B_t-c_t.
\label{eq:wallet-transition}
\end{equation}

Outcomes and realized costs are observed only for the selected provider,
yielding \emph{chosen-only paid feedback}. Provider outcome distributions and
charges may change over time without announced change points.

\subsection{Buyer Objective}
\label{sec:buyer-objective}

The realized service utility is
\begin{equation}
u_t=q_t-\nu f_t,
\label{eq:service-utility}
\end{equation}
where $\nu\ge0$ penalizes paid non-delivery. The realized charge is normalized
as
\begin{equation}
\widetilde c_t
=
\operatorname{clip}
\left(
\frac{c_t}{c_{\max}},
0,
1
\right),
\qquad
c_{\max}>0,
\label{eq:normalized-charge}
\end{equation}
where $c_{\max}$ is a fixed normalization constant.

Let $\lambda_t>0$ denote wallet pressure, computed from the current balance
relative to a spending plan. It rises when cumulative spending runs ahead of
plan and falls when spending runs behind. Its normalized value is
\begin{equation}
\lambda_{\mathrm{norm},t}
=
\frac{\lambda_t}{1+\lambda_t}
\in(0,1).
\label{eq:normalized-wallet-pressure}
\end{equation}
The payment-aware reward for a completed purchase is
\begin{equation}
r_t
=
\left(
1-\lambda_{\mathrm{norm},t}
\right)
u_t
-
\lambda_{\mathrm{norm},t}
\widetilde c_t.
\label{eq:payment-aware-reward}
\end{equation}
Thus, increasing wallet pressure shifts weight from service utility to payment.

The buyer seeks an admissible policy maximizing expected cumulative
payment-aware reward over the served rounds:
\begin{equation}
\pi^\star
\in
\arg\max_{\pi\in\Pi_{\mathrm{adm}}}
\;
\mathbb E^{\pi}
\left[
\sum_{t=0}^{\tau-1}
r_t
\right].
\label{eq:buyer-objective}
\end{equation}
This formulation is independent of any particular decision or learning
algorithm.

\section{The 402Pilot Decision Layer}
\label{sec:layer}

To operationalize this decision process, we introduce 402Pilot, a
policy-agnostic layer between agent buyers and the payment stack. It implements
provider selection and post-payment learning through interchangeable buyer
policies. Figure~\ref{fig:architecture} shows the per-round decision flow.

\begin{figure*}[!t]
\centering
\includegraphics[width=\linewidth]{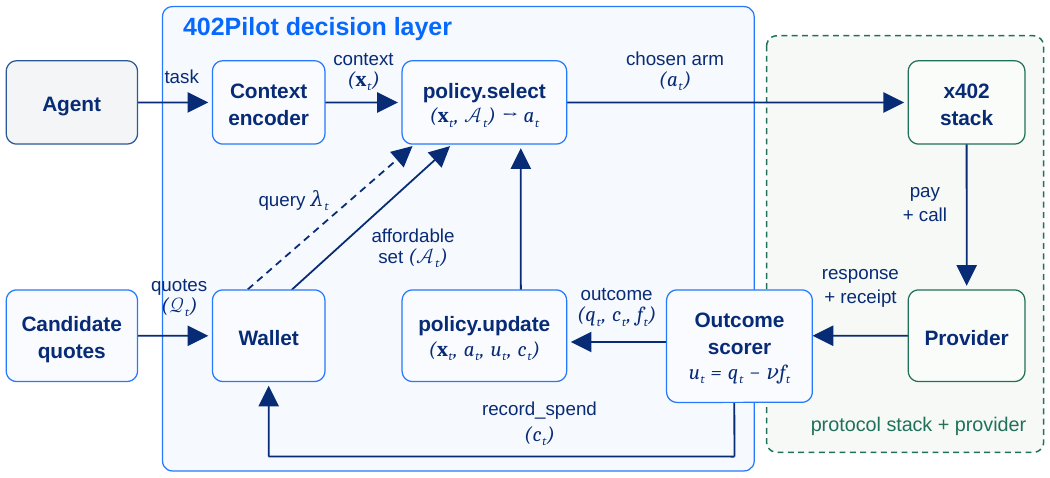}
\caption{The 402Pilot per-round decision flow. Upstream components supply a
task and payable candidates $\mathcal{Q}_t$. Within 402Pilot, the context
encoder derives the task context $\mathbf{x}_t$, while the wallet separately
exposes the affordable set $\mathcal{A}_t$ and wallet pressure $\lambda_t$.
The policy selects $a_t$; the payment substrate executes the selected
transaction; and the returned response and receipt yield $(q_t,f_t,c_t)$.
The layer then reports $c_t$ to the wallet and passes
$(\mathbf{x}_t,a_t,u_t,c_t)$ to the policy update. The figure shows x402 as
the payment substrate, but the same boundary can be implemented by another
programmable-payment substrate.}
\Description{Architecture diagram showing separate task-context and wallet
inputs to provider selection, followed by payment execution, outcome scoring,
wallet updating, and chosen-only policy feedback.}
\label{fig:architecture}
\end{figure*}

Upstream components provide a task request and candidate payment requirements
\[
\mathcal{Q}_t=\{(a,p_{a,t})\}_{a=1}^{K}.
\]
The context encoder maps the request to $\mathbf{x}_t$, while the wallet uses
$\mathcal{Q}_t$ to derive the affordable set $\mathcal{A}_t$ and wallet
pressure $\lambda_t$. Wallet pressure compares cumulative spending with the
expected budget trajectory:
\begin{equation}
\lambda_t
=
\begin{cases}
\lambda_0, & t=0,\\[2mm]
\lambda_0
\exp\!\left[
\alpha\left(
\dfrac{(B-B_t)/B}{t/T}-1
\right)
\right],
& t>0,
\end{cases}
\label{eq:wallet-pressure-schedule}
\end{equation}
where $\lambda_0>0$ is the baseline pressure and $\alpha\geq0$ controls its
sensitivity to deviations from the spending plan. The exponential form keeps
pressure positive and smoothly increases cost sensitivity as overspending
grows. The policy then invokes
$\mathtt{select}(\mathbf{x}_t,\mathcal{A}_t)\to a_t\in\mathcal{A}_t$ and passes
the selected provider to the underlying payment substrate, such as x402, for
execution and settlement. After the transaction, the outcome scorer converts
the provider response and payment receipt into the realized utility $u_t$ and
charge $c_t$. The wallet records $c_t$, and the policy receives the
post-payment feedback through
$\mathtt{update}(\mathbf{x}_t,a_t,u_t,c_t)$.

The interface restricts selection to $\mathcal{A}_t$ and exposes feedback only
for the chosen provider, consistent with the affordability and chosen-only
feedback setting in Section~\ref{sec:problem}. It supports diverse buyer
policies, including heuristic, optimization-based, and learning-based
approaches.

\section{The PA-DCT Policy}
\label{sec:method}

As a concrete buyer policy within 402Pilot, we introduce Payment-Aware
Discounted Contextual Thompson Sampling (PA-DCT), a contextual bandit. PA-DCT
operates as a continual decision loop: it maintains discounted Bayesian
posteriors over provider utility and cost, draws Thompson samples, and uses
wallet pressure to select among affordable providers. After each transaction,
realized utility and cost update the posteriors, allowing the policy to adapt
as provider performance and prices change.

\subsection{Discounted Utility and Cost Posteriors}
\label{sec:method-posteriors}

PA-DCT conditions its estimates on task context through a discrete bucket map
$k_t=k(\mathbf{x}_t)\in\{1,\ldots,K_b\}$. For each provider--context pair
$(a,k)$, it maintains two discounted Bayesian posteriors: a Q-posterior over
realized utility $u_t=q_t-\nu f_t$ and a C-posterior over realized cost $c_t$.

Each cell stores a shared effective count $n^{a,k}_t$, a utility sum
$S^{a,k}_t$, and a cost sum ${S_c}^{a,k}_t$. Each completed transaction
contributes one observation to both posteriors. To remain responsive to
changing reliability and prices, PA-DCT applies a common discount factor
$\gamma\in(0,1]$ to every cell at the start of each round. The selection-time
statistics are
\begin{equation}
\label{eq:discounted-statistics}
\begin{aligned}
\bar n^{a,k}_t &= \gamma n^{a,k}_{t-1},\\
\bar S^{a,k}_t &= \gamma S^{a,k}_{t-1},\\
\bar S^{a,k}_{c,t} &= \gamma {S_c}^{a,k}_{t-1}.
\end{aligned}
\end{equation}
The Q-posterior uses prior mean $\mu_{0,q}$, prior variance
$\sigma^2_{0,q}$, and likelihood variance $\sigma^2_q$. Its posterior variance
and mean are
\begin{equation}
\label{eq:posterior-q}
\begin{aligned}
\Sigmaq{a}{k}
&= \left(\frac{1}{\sigma^2_{0,q}}
  + \frac{\bar n^{a,k}_t}{\sigma^2_q}\right)^{-1},\\
\thetaq{a}{k}
&= \Sigmaq{a}{k}
  \left(\frac{\mu_{0,q}}{\sigma^2_{0,q}}
  + \frac{\bar S^{a,k}_t}{\sigma^2_q}\right).
\end{aligned}
\end{equation}
where $\thetaq{a}{k}$ and $\Sigmaq{a}{k}$ denote the posterior mean and
variance of the utility model, respectively.

The C-posterior follows the same Normal--Normal construction. Its posterior
mean and variance $(\thetac{a}{k},\Sigmac{a}{k})$ are computed analogously
using prior $\mathcal{N}(\bar c_a,\sigma^2_{0,c})$, likelihood variance
$\sigma^2_c$, and discounted cost sum $\bar S^{a,k}_{c,t}$. The initial
specification cost $\bar c_a$ defines the cost prior, while observed receipt
charges continually update the posterior as provider prices change.

All provider--context cells are discounted once per round, including those not
selected in the current round, so stale evidence gradually loses influence
over time.

\subsection{Selection Rule}
\label{sec:method-decision}

At decision time, PA-DCT implements $\mathtt{select}$ using the task context
$\mathbf{x}_t$, the affordable provider set $\mathcal{A}_t$, and the wallet
pressure $\lambda_t$. Write $\lamn=\lambda_{\mathrm{norm},t}$ for the
normalized wallet pressure in Eq.~\eqref{eq:normalized-wallet-pressure}. The
context bucket determines the corresponding provider--context cells, from
which Thompson samples are drawn as
\begin{equation}
\label{eq:padct-sample}
\begin{aligned}
\hatu{a}{k_t}
&\sim \mathcal{N}\!\big(\thetaq{a}{k_t},\, \Sigmaq{a}{k_t}\big),\\
\hatc{a}{k_t}
&\sim \mathcal{N}\!\big(\thetac{a}{k_t},\, \Sigmac{a}{k_t}\big).
\end{aligned}
\end{equation}
The sampled cost is first normalized, after which the sampled utility and
normalized cost are combined into the following payment-aware reward:
\begin{equation}
\label{eq:sampled-pa-reward}
\begin{aligned}
\widetilde{\hat c}^{a,k_t}
&= \operatorname{clip}\!\left(
  \hatc{a}{k_t}/c_{\max},0,1\right),\\
\hat r_a
&= (1-\lamn)\hatu{a}{k_t}
  - \lamn\widetilde{\hat c}^{a,k_t}.
\end{aligned}
\end{equation}
The sampled reward in Eq.~\eqref{eq:sampled-pa-reward} is the pre-payment
counterpart of the realized reward in Eq.~\eqref{eq:payment-aware-reward}: it
replaces realized utility and normalized cost with Thompson samples. As wallet
pressure increases, selection becomes progressively more cost-sensitive while
remaining driven by posterior samples rather than fixed cost thresholds. The
provider with the highest sampled reward is selected:
$a_t=\arg\max_{a\in\mathcal{A}_t}\hat r_a$.

\begin{algorithm}[tb]
\caption{PA-DCT policy loop.}
\label{alg:padct}
\begin{algorithmic}[1]
\Require Initial specification costs $\bar c_a$; discount $\gamma$;
normalizer $c_{\max}$.
\Require Context encoder $k(\cdot)$ and Q/C posterior priors.
\For{$(a,k)\in\{1,\ldots,K\}\times\{1,\ldots,K_b\}$}
  \State $n^{a,k}\leftarrow 0$;\quad $S^{a,k}\leftarrow 0$;\quad
  ${S_c}^{a,k}\leftarrow 0$
\EndFor
\For{each served round $t$ with $\mathcal{A}_t\neq\emptyset$}
  \State Receive $\mathbf{x}_t$ and $\mathcal{A}_t$
  \State For all $(a,k)$, multiply $n^{a,k}$, $S^{a,k}$, and
  ${S_c}^{a,k}$ by $\gamma$
  \State Read $\lambda_t$;\quad $k_t\leftarrow k(\mathbf{x}_t)$;
  \quad $\lamn\leftarrow\lambda_t/(1+\lambda_t)$
  \For{$a\in\mathcal{A}_t$}
    \State Compute the Q/C posterior parameters for $(a,k_t)$
    \State Sample $\hatu{a}{k_t}$ and $\hatc{a}{k_t}$ using
    Eq.~\eqref{eq:padct-sample}
    \State $\widetilde{\hat c}^{a,k_t}\leftarrow
    \operatorname{clip}(\hatc{a}{k_t}/c_{\max},0,1)$
    \State $\hat r_a\leftarrow(1-\lamn)\hatu{a}{k_t}
    -\lamn\widetilde{\hat c}^{a,k_t}$
  \EndFor
  \State Select $a_t\leftarrow\arg\max_{a\in\mathcal{A}_t}\hat r_a$
  \State Receive realized $(\ut,c_t)$
  \State $n^{a_t,k_t}\mathrel{+}=1$;\quad
  $S^{a_t,k_t}\mathrel{+}=\ut$;\quad
  ${S_c}^{a_t,k_t}\mathrel{+}=c_t$
\EndFor
\end{algorithmic}
\end{algorithm}

\subsection{Learning from Post-Payment Feedback}
\label{sec:method-feedback}

After the selected transaction completes, PA-DCT implements $\mathtt{update}$
using the realized utility $\ut$ and receipt cost $c_t$. The task is mapped to
the corresponding context bucket $k_t$, and the observations are incorporated
only into the selected provider--context cell:
\begin{equation}
\label{eq:chosen-cell-update}
\begin{aligned}
n^{a_t,k_t}_t &= \bar n^{a_t,k_t}_t + 1,\\
S^{a_t,k_t}_t &= \bar S^{a_t,k_t}_t + \ut,\\
{S_c}^{a_t,k_t}_t &= \bar S^{a_t,k_t}_{c,t} + c_t.
\end{aligned}
\end{equation}
All other provider--context cells retain their discounted statistics from
Eq.~\eqref{eq:discounted-statistics}.

Algorithm~\ref{alg:padct} summarizes the PA-DCT policy loop.
Appendix~\ref{appendix:proofs} provides formal support for its key components
and discusses the scope of the analysis.

\section{Evaluation}
\label{sec:eval}

Evaluating buyer-side payment policies requires a benchmark
with diverse tasks, heterogeneous providers, controllable
market scenarios, and reproducible execution. Existing LLM
benchmarks do not jointly satisfy these requirements, so we
construct 402Pilot-Bench by extending public datasets with
multiple provider pipelines, prices, and market dynamics.
Under its frozen-replay design, multiple responses are
pre-generated and scored for each task--provider pair, then
sampled from the same frozen pool across policies, enabling
reproducible comparisons without live-inference variability.
Since frozen replay abstracts payment execution, a local x402
witness separately validates the decision-layer interface
against an actual HTTP~402 payment flow; implementation
details appear in Appendix~\ref{appendix:x402-witness}.

\subsection{Experimental Setup}
\label{sec:eval-setup}

\paragraph{Benchmark.}
\emph{402Pilot-Bench} contains $823$ effective tasks across the
four datasets in Table~\ref{tab:benchmark}~\citep{chen2021humaneval,
yang2018hotpotqa,joshi2017triviaqa,kopf2023openassistant}, which define the
task buckets used by PA-DCT. Utility scores are normalized to
$[0,1]$ using the dataset-specific metrics shown in the table. For each
task--provider pair, we cache five independently generated
responses with precomputed utility scores. Across the five providers
described below, this yields $823$ tasks $\times 5$ providers
$\times 5$ responses $=20{,}575$ scored responses. All scenarios
replay the same task pool and response cache; only realized failures
and effective prices change.

\begin{table}[!htbp]
\centering
\small
\setlength{\tabcolsep}{3.5pt}
\begin{tabularx}{\columnwidth}{@{}>{\raggedright\arraybackslash}p{0.34\columnwidth}r>{\raggedright\arraybackslash}X@{}}
\toprule
\textbf{Dataset / task type} & \textbf{Tasks}
& \textbf{Utility metric} \\
\midrule
HumanEval / coding
& 164 & Execution-based pass@1 \\
HotpotQA / multi-hop QA
& 220 & Max normalized EM/F1 \\
TriviaQA-web / closed-form web QA
& 219 & Max normalized EM/F1 \\
OpenAssistant / open-ended QA
& 220 & Cached LLM-as-judge \\
\midrule
\textbf{Total}
& \textbf{823} & -- \\
\bottomrule
\end{tabularx}
\caption{Composition of 402Pilot-Bench.}
\label{tab:benchmark}
\end{table}

\paragraph{Providers.}
Table~\ref{tab:providers} summarizes the five providers.
P-cheap, P-mid, and P-premium represent different
price--quality levels. P-adv and P-flaky simulate two
problems that buyers may face in an open service market:
fraudulent or misrepresented service quality, and unstable
service that may fail after payment. Because both have the
same listed price as P-mid, their true performance must be
learned from observed outcomes.

\begin{table}[!htbp]
\centering
\small
\setlength{\tabcolsep}{2.5pt}
\begin{tabularx}{\linewidth}{@{}lcX@{}}
\toprule
Provider & Price & Pipeline \\
\midrule
P-cheap   & $0.0005$ & qwen3.5-flash; uniform prompt; \emph{cheap tier} \\
P-mid     & $0.002$  & GPT-5.4-mini; uniform prompt; \emph{mid tier} \\
P-premium & $0.01$   & GPT-5.4; uniform prompt; \emph{premium tier} \\
P-adv     & $0.002$  & GPT-5.4-mini; 60\% adversarial prompt; \emph{fluent wrong answers} \\
P-flaky   & $0.002$  & GPT-5.4-mini; 40\% timeout injection; \emph{billed timeouts} \\
\bottomrule
\end{tabularx}
\caption{Provider set used in 402Pilot-Bench. P-mid, P-adv, and P-flaky share the same price tier and base model.}
\label{tab:providers}
\end{table}

\paragraph{Scenario design.}
The scenarios test whether a policy can maintain a useful
allocation when provider conditions remain stable or change
without warning. S1 is stationary. In S2, P-mid experiences
a reliability shock, with a 30\% forced-failure rate from
rounds 3,000 to 5,500 before recovery. In S3, P-premium's
price drops from \$0.01 to \$0.002 at round 1,000, testing
whether the policy shifts toward a newly attractive provider.

Each policy is evaluated for \(T=10{,}000\) rounds with a
\$50 budget over 30 paired seeds. This budget is deliberately
binding: at the initial premium price, it covers only 5,000
of the 10,000 rounds. For each seed, all policies are evaluated
using the same task order, cached-response draws, and scenario
changes.

\paragraph{Hyperparameters.}
The discount factor $\gamma$, paid-non-delivery penalty $\nu$,
and cost scale $c_{\max}$ most directly govern the policy's
adaptation, failure penalty, and cost normalization. The discount factor
$\gamma=0.999$ provides an effective memory of about
$10^3$ rounds, balancing responsiveness with stable
estimation. The paid-non-delivery penalty $\nu=0.5$
makes billed failures materially worse without overwhelming
the $[0,1]$ utility signal. The cost scale
$c_{\max}=\$0.01$ maps the initial premium price to 1 and
keeps all lower prices on the same normalized scale.
All remaining settings are fixed; complete notation and
values appear in Appendix~\ref{appendix:notation}, with sensitivity analyses in
Appendix~\ref{appendix:sensitivity}.

\paragraph{Comparators.}
The comparators cover non-learning policies, learning
policies, and an oracle upper bound. The non-learning group
includes Random, three fixed-arm policies, and BudgetRule,
capturing uninformed selection, static provider preferences,
and a hand-crafted wallet rule. BudgetRule switches from
P-premium to P-mid and then P-cheap when the remaining
budget falls below 50\% and 20\%, respectively.

The learning group covers the main capabilities relevant to
buyer-side routing. Contextual DS-TS~\citep{qi2023dsts}
tracks utility drift but lacks wallet pressure; Contextual
BTS~\citep{xia2015bts} learns utility and cost under a finite
budget but does not discount stale observations; PM-Greedy,
inspired by FrugalGPT-style routing~\citep{chen2023frugalgpt},
routes using listed prices rather than learned realized
costs; and LinCBwK-Adapt~\citep{agrawal2016linearcbwk}
combines context and budget constraints but lacks explicit
drift discounting.

The replay True Oracle selects the affordable arm with the
highest realized payment-aware reward and is reported only
as an unattainable upper bound. Implementation details and
replay adaptations of all comparators appear in
Appendix~\ref{sec:app-extended-baselines}.

\subsection{Metrics}
\label{sec:eval-metrics}

Evaluating buyer-side decisions under a fixed wallet requires
considering both the service quality delivered and how the
budget is used. We therefore report complementary metrics
covering delivered quality, wallet consumption, spending
efficiency, and performance relative to the replay upper
bound. Table~\ref{tab:metrics} gives their calculations.

Full-horizon mean quality measures the service quality
delivered over all \(T\) rounds, with rounds after budget
exhaustion counted as zero. Budget use captures the share of
the initial wallet consumed. ROI measures the quality obtained
per dollar spent and is used as an efficiency diagnostic rather
than an optimization objective. PA-gap/\(T\) measures the
per-round cumulative payment-aware reward gap relative to the
replay True Oracle, providing an empirical replay counterpart
of the buyer objective in Eq.~\eqref{eq:buyer-objective}.

\begin{table}[!htbp]
\centering
\small
\setlength{\tabcolsep}{2pt}
\begin{tabularx}{\linewidth}{@{}lXX@{}}
\toprule
Metric & Calculation & Meaning \\
\midrule
$\bar q_T \uparrow$ & $(1/T)\sum_{t<\tau}q_t$ & Mean quality; exhausted rounds count as zero. \\
Budget use & $100\sum_{t<\tau}c_t/B$ & Wallet share spent. \\
ROI $\uparrow$ & $\sum_{t<\tau}q_t/\sum_{t<\tau}c_t$ & Quality per dollar; efficiency only. \\
PA-gap/$T$ $\downarrow$ & $(R^{\mathrm{TO}}_T-R^\pi_T)/T$ & Gap to replay True Oracle. \\
\bottomrule
\end{tabularx}
\caption{Evaluation metrics.}
\label{tab:metrics}
\end{table}

\subsection{Main Results}
\label{sec:eval-main}

\begin{table*}[!t]
\centering
\small
\setlength{\tabcolsep}{2pt}
\begin{tabular*}{\textwidth}{@{\extracolsep{\fill}}lrrrrrrrrrrrr@{}}
\toprule
& \multicolumn{4}{c}{S1 Stationary}
& \multicolumn{4}{c}{S2 Mid Outage}
& \multicolumn{4}{c}{S3 Premium Promo} \\
\cmidrule(lr){2-5}\cmidrule(lr){6-9}\cmidrule(lr){10-13}
Policy
& $\bar q_T \uparrow$ & \shortstack{Budget\\use (\%)} & ROI $\uparrow$ & \shortstack{PA-gap/$T$\\$\downarrow$}
& $\bar q_T \uparrow$ & \shortstack{Budget\\use (\%)} & ROI $\uparrow$ & \shortstack{PA-gap/$T$\\$\downarrow$}
& $\bar q_T \uparrow$ & \shortstack{Budget\\use (\%)} & ROI $\uparrow$ & \shortstack{PA-gap/$T$\\$\downarrow$} \\
\midrule
Random            & $0.688$ & $66$  & $209$  & $0.365$ & $0.676$ & $66$  & $205$  & $0.375$ & $0.688$ & $37$ & $370$  & $0.274$ \\
Always-P-cheap    & $0.610$ & $10$  & $\mathbf{1{,}220}$ & $0.167$ & $0.610$ & $10$  & $\mathbf{1{,}220}$ & $\mathbf{0.164}$ & $0.610$ & $10$ & $\mathbf{1{,}220}$ & $0.195$ \\
Always-P-mid      & $0.819$ & $40$  & $410$  & $\mathbf{0.101}$ & $0.757$ & $40$  & $379$  & $0.174$ & $0.819$ & $40$ & $410$  & $0.129$ \\
Always-P-premium  & $0.433$ & $100$ & $87$   & $1.072$ & $0.433$ & $100$ & $87$   & $1.070$ & $\mathbf{0.865}$ & $56$ & $309$  & $0.400$ \\
BudgetRule        & $\mathbf{0.831}$ & $80$ & $208$ & $0.692$ & $0.769$ & $80$ & $192$ & $0.722$ & $0.859$ & $56$ & $307$ & $0.405$ \\
\addlinespace[2pt]
Contextual DS-TS  & $0.559$ & $100$ & $112$ & $0.858$ & $0.514$ & $100$ & $103$ & $0.884$ & $0.840$ & $48$ & $349$ & $0.264$ \\
Contextual BTS    & $0.612$ & $10$  & $1{,}187$ & $0.169$ & $0.612$ & $10$  & $1{,}187$ & $0.166$ & $0.612$ & $10$ & $1{,}187$ & $0.197$ \\
PM-Greedy         & $0.732$ & $90$  & $179$ & $0.542$ & $0.627$ & $99$  & $128$ & $0.660$ & $0.836$ & $44$ & $384$ & $0.198$ \\
LinCBwK-Adapt     & $0.824$ & $93$  & $180$ & $0.457$ & $\mathbf{0.796}$ & $100$ & $159$ & $0.515$ & $0.836$ & $45$ & $377$ & $0.195$ \\
\addlinespace[2pt]
\textbf{PA-DCT (ours)} & $0.797$ & $42$ & $378$ & $0.133$ & $0.761$ & $43$ & $357$ & $0.166$ & $0.831$ & $39$ & $429$ & $\mathbf{0.121}$ \\
\midrule
True Oracle (UB)  & $0.901$ & $32$ & $561$ & $0.000$ & $0.901$ & $33$ & $551$ & $0.000$ & $0.906$ & $25$ & $722$ & $0.000$ \\
\bottomrule
\end{tabular*}
\caption{Main results on 402Pilot-Bench, averaged over $30$ paired seeds. Rows are grouped as non-learning baselines, learning baselines, PA-DCT, and the replay upper bound. Bold marks the best non-oracle entry in the quality, ROI, and PA-gap/$T$ columns.}
\label{tab:main}
\end{table*}

PA-DCT delivers the strongest fixed-wallet adaptive trade-off on 402Pilot-Bench: it maintains competitive service quality under controlled spend and reallocates in the right direction when reliability or prices change. Across the three scenarios in Table~\ref{tab:main}, PA-DCT consistently uses only $39$--$43\%$ of the wallet while remaining among the strongest non-oracle policies on PA-gap/$T$, including the best non-oracle result in S3. Together, these results show that the objective of buyer-side payment decision-making is not to maximize service quality in every market, but to learn when additional quality is worth paying for under a finite wallet. Accordingly, PA-DCT is not intended to maximize pointwise quality in every scenario: policies such as LinCBwK-Adapt and premium-heavy rules occasionally achieve higher absolute quality, but only by consuming substantially more of the wallet, resulting in weaker payment-aware trade-offs.

The three scenarios highlight different aspects of PA-DCT's behavior. In the stationary market (S1), Always-P-mid achieves the best non-oracle PA-gap/$T$ with similarly low budget use, reflecting the advantage of a fixed policy when its selected tier matches a stable market. PA-DCT pays a modest exploration cost but still uses only $42\%$ of the wallet; by contrast, BudgetRule and LinCBwK-Adapt obtain slightly higher quality while consuming $80\%$ and $93\%$, respectively. This highlights the role of wallet pressure in preserving long-horizon budget efficiency. In the reliability shock (S2), PA-DCT is not significantly different from the numerically best non-oracle baseline, Always-P-cheap, on PA-gap/$T$, while delivering substantially higher service quality ($0.761$ vs.\ $0.610$) through reallocation away from the failing mid-tier provider. This highlights the role of reliability adaptation in preserving service quality. In the price shock (S3), PA-DCT achieves the best non-oracle PA-gap/$T$. Several learning baselines attain slightly higher absolute quality, but their gains come with substantially worse PA-gap/$T$ ($0.195$--$0.264$ vs.\ $0.121$) and lower ROI. This highlights the role of price adaptation in reallocating purchases as provider prices change while keeping spending under control.

Appendix~\ref{appendix:allocation-dynamics} provides a direct view of these allocation shifts, showing PA-DCT moving away from P-mid during the outage, returning after recovery, and shifting toward discounted P-premium after the price promotion.

A closer look at the three types of baselines shows how their underlying design principles are reflected in the experimental results. Among the non-learning policies, fixed-arm strategies perform well only when their selected tier matches the market, while BudgetRule follows wallet thresholds rather than provider feedback. The Thompson-sampling baselines expose a split between adaptation and budget control. Contextual DS-TS responds to market changes but exhausts the wallet in S1 and S2, whereas Contextual BTS preserves the wallet but remains near the cheap tier and therefore misses the discounted premium opportunity in S3. The cost-aware learning baselines provide the strongest quality comparison, but not the strongest payment-aware trade-off. PM-Greedy and LinCBwK-Adapt achieve competitive quality, but consistently at the cost of substantially higher wallet consumption and worse PA-gap/$T$. Taken together, these results show that none of the baselines simultaneously maintains budget discipline and adapts effectively across reliability and price changes. The ablation study below separates the roles of PA-DCT's components in achieving this balance.

\subsection{Ablation Study}
\label{sec:eval-ablations}

The ablations further clarify the role of each PA-DCT component. To complement
the endpoint metrics, AdaptT measures the number of post-shock rounds required
for sustained recovery (lower is better); complete ablation definitions and
results appear in Appendix~\ref{appendix:ablations}. Payment-aware ranking primarily preserves
full-horizon solvency: removing it ($-P$) exhausts every S1/S2 wallet and
increases PA-gap/$T$ from $0.166$ to $0.878$ in S2, while full-horizon mean
quality falls from $0.761$ to $0.517$. Adaptation depends on both temporal discounting and
realized-cost learning. Removing discounting ($-D$) slows S2 recovery by $53\%$
($1{,}467\!\rightarrow\!2{,}249$ rounds), whereas replacing learned realized
costs with specification prices ($-C_{\mathrm{post}}$) largely eliminates price
adaptation in S3, reducing post-shock P-premium selection from $66.1\%$ to $3.6\%$ and
increasing PA-gap/$T$ from $0.121$ to $0.144$.

Context and Thompson sampling mainly improve adaptation efficiency and
robustness. Removing contextual bucketing ($-C$) delays S3 AdaptT from $200$ to
$398$ rounds with little change in endpoint performance. Replacing Thompson
sampling with posterior means ($-\mathrm{TS}$) increases the S1 PA-gap/$T$
standard deviation from $0.004$ to $0.046$ and delays S3 AdaptT from $200$ to $2{,}232$
rounds, with only $24/30$ seeds reaching the AdaptT threshold. Overall, the ablations
assign distinct operational roles to the components: payment awareness
preserves solvency, discounting and realized-cost learning enable adaptation,
context accelerates recovery, and Thompson sampling improves robustness.

\section{Discussion and Future Work}
\label{sec:discussion}

Recent advances in programmable payment protocols and community-driven
reputation mechanisms provide an increasingly capable infrastructure for
agentic commerce. Our results suggest that infrastructure alone is not
sufficient for autonomous payment, because agents must still decide how to
allocate a limited budget as service value and market conditions evolve. This
motivates buyer-side decision-making as a distinct layer between agents and
payment protocols. PA-DCT serves as an initial instantiation of this decision
layer rather than a definitive solution; future work may explore reinforcement
learning, foundation-model-based policies, or other decision mechanisms within
the same framework.

Our evaluation emphasizes reproducible policy comparison rather than complete
deployment validation. Frozen replay isolates buyer-side decision quality under
identical service traces, while the local x402 witness demonstrates end-to-end
integration with an HTTP~402 payment loop. Extending evaluation to controlled
live-market deployments remains important future work.

The benchmark also covers a finite set of providers, tasks, and market
conditions, and learning quality ultimately depends on the reliability of
outcome feedback. Future work should expand the benchmark to broader industrial
scenarios, richer provider ecosystems, more complex market dynamics, and more
reliable outcome evaluation.

\section{Conclusion}
\label{sec:conclusion}

Programmable payment protocols enable agents to execute payments, but they do
not determine what is worth paying for. We introduced 402Pilot as a buyer-side
decision layer between agents and payment protocols, together with PA-DCT as an
initial payment-aware decision policy for budget-constrained service selection.
Evaluation on 402Pilot-Bench showed that PA-DCT achieves a strong fixed-wallet
trade-off while adapting its purchasing decisions under reliability and price
changes. These results establish buyer-side decision-making as a new research
direction for agentic commerce and motivate future work on richer decision
policies, broader benchmarks, and real-world deployment.

\FloatBarrier
\putbib[references]
\end{bibunit}

\clearpage
\appendix
\begin{bibunit}[aaai2027]
\section{Notation and Locked Settings}
\label{appendix:notation}

This appendix collects the notation and locked experimental settings used
throughout the paper and Appendix~B.

\begin{table*}[!t]
\centering
\small
\setlength{\tabcolsep}{5pt}
\begin{tabularx}{\textwidth}{@{}l>{\raggedright\arraybackslash}X>{\raggedright\arraybackslash}X@{}}
\toprule
Symbol & Meaning & Domain / locked value \\
\midrule
\multicolumn{3}{@{}l}{\textbf{Rounds, arms, and context}} \\
$t,T$ & Round index and horizon & $t\in\{0,\ldots,T-1\}$; $T=10{,}000$ \\
$a,K$ & Provider arm and number of arms & $K=5$: P-cheap, P-mid, P-premium, P-adv, P-flaky \\
$a_t$ & Arm selected at round $t$ & $a_t\in\mathcal{A}_t$ \\
$\mathbf{x}_t$ & Task context & Released encoder: $d=7$ (task-type one-hot, difficulty, and two wallet features drawn from $H_t$) \\
$H_t$ & Pre-selection information available to the policy & Context, wallet state, affordable set, exposed provider metadata, and past purchase feedback \\
$k(\mathbf{x}_t),K_b$ & Task-type bucket map used by \padct{} & T1/T2/T3a/T3b; $K_b=4$ \\
\addlinespace[2pt]
\multicolumn{3}{@{}l}{\textbf{Payment loop and feedback}} \\
$\mathcal{Q}_t$ & Quoted candidate set from discovery/pricing & Provider--price pairs before selection \\
$p_{a,t}$ & Provider charge at round $t$ & May drift; used for affordability and may be exposed as listed-price metadata \\
$\mathcal{A}_t$ & Affordable set, exposed as a mask before selection & $\mathcal{A}_t=\{a:p_{a,t}\le B_t\}$ \\
$q_t$ & Observed quality of the chosen arm & $[0,1]$ \\
$f_t$ & Paid non-delivery flag (timeout, parse failure) & $\{0,1\}$ \\
$c_t$ & Realized charge & $c_t=p_{a_t,t}$, in USDC \\
$B,B_t$ & Initial and remaining budget & $B=\$50$; $B_{t+1}=B_t-c_t$ \\
$\tau$ & Completed purchases before termination or horizon end & $\tau\le T$ \\
\addlinespace[2pt]
\multicolumn{3}{@{}l}{\textbf{Reward and wallet pressure}} \\
$\nu$ & Failure penalty in utility & $\nu=0.5$ \\
$\ut$ & Utility: failure-penalized quality; updates the Q-posterior & $\ut=q_t-\nu f_t$ \\
$c_{\max},\ctilde$ & Cost normalizer and normalized cost & $c_{\max}=\$0.01$; $\ctilde=\operatorname{clip}(c_t/c_{\max},0,1)$ \\
$\bdev_t$ & Deviation of the realized spend rate from the linear plan & $\bdev_t=\frac{(B-B_t)/B}{t/T}-1$ for $t>0$; $\bdev_0=0$ \\
$\alpha,\lambda_0$ & Budget-pressure schedule constants & $\alpha=2.0$, $\lambda_0=1$ \\
$\lambda_t,\lamn$ & Raw and normalized wallet-pressure weights & $\lambda_t=\lambda_0\exp(\alpha\,\bdev_t)$; $\lamn=\lambda_t/(1+\lambda_t)$ \\
$\rt$ & Payment-aware reward & $\rt=(1-\lamn)\ut-\lamn\ctilde$ \\
\addlinespace[2pt]
\multicolumn{3}{@{}l}{\textbf{\padct{} posteriors and selection}} \\
$\bar c_a$ & Initially listed cost of arm $a$ & Prior mean of the C-posterior \\
$\mu_{0,q},\sigma^2_{0,q},\sigma^2_q$ & Q-prior mean, prior variance, and likelihood variance & $0.5$, $1.0$, $0.09$ \\
$\sigma^2_{0,c},\sigma^2_c$ & C-prior variance and likelihood variance & $10^{-4}$, $10^{-6}$ \\
$n^{a,k}_t,\bar n^{a,k}_t$ & Post-feedback and selection-time discounted counts of cell $(a,k)$ & Shared by the Q- and C-posteriors \\
$S^{a,k}_t,{S_c}^{a,k}_t$ & Discounted utility and cost sums (barred: selection-time) & New observations update only the chosen arm--bucket cell \\
$\gamma$ & Common discount of the Q- and C-posteriors & $0.999$ in main results \\
$\thetaq{a}{k},\Sigmaq{a}{k}$ & Q-posterior mean and variance & Normal--Normal posterior over utility \\
$\thetac{a}{k},\Sigmac{a}{k}$ & C-posterior mean and variance & Normal--Normal posterior over realized cost \\
$\hatu{a}{k},\hatc{a}{k}$ & Thompson samples drawn at decision time & From the Q- and C-posteriors \\
$\hat r_a$ & Decision-time sampled payment-aware reward & Mirrors $\rt$ with sampled utility and clipped sampled cost \\
\addlinespace[2pt]
\multicolumn{3}{@{}l}{\textbf{Reported evaluation quantities}} \\
$\bar q_T$ & Full-horizon mean quality & Exhausted rounds count as $q_t=0$ \\
$R_T^\pi,\paReg_T$ & Cumulative PA reward and empirical PA-regret & $\paReg_T=R_T^{\mathrm{Oracle}}-R_T^\pi$; reported as PA-gap/$T$ \\
$\mathrm{ROI}$ & Quality per dollar & Efficiency diagnostic only \\
$\mathrm{AdaptT}$ & Shock-response diagnostic in S2/S3 & Rounds to the first qualifying ROI window; defined in Appendix~\ref{appendix:ablations} \\
\bottomrule
\end{tabularx}
\caption{Compact notation legend. Definitions remain in the main text; this
table records conventions and the locked values used in all reported
experiments.}
\label{tab:notation-compact}
\end{table*}

\section{Proofs}
\label{appendix:proofs}

This appendix provides formal support for the main analytical properties of
\padct{} rather than a complete analysis of the replay system. Specifically,
we establish reward boundedness (Prop.~\ref{prop:reward-bounded}), posterior
consistency in the stationary no-discount regime
(Prop.~\ref{prop:consistency}), a stationary well-specified regret reduction
(Thm.~\ref{thm:stationary-regret}), and a fixed-schedule adaptation rate after
an abrupt change (Thm.~\ref{thm:adaptation}). The regret theorem is
intentionally stated for a clean version of the score;
\S\ref{app:proof-limitations} records the gap to the full empirical setting.

\subsection{Shared assumptions}
\label{app:proof-setup}

\ifdefined\technicalreportonly
We use the notation in Appendix~\ref{appendix:notation} and the main report.
\else
We use the notation of \S\ref{sec:problem} and \S\ref{sec:method} of the main
paper, summarized in Appendix~\ref{appendix:notation}.
\fi
All regret sums below are over served rounds with nonempty affordable set; if
a wallet exhausts early, replace $T$ by the served horizon $\tau$. The formal
claims use the following assumptions.
Different results require different subsets of these assumptions; only the
stationary regret theorem relies on the stronger assumptions
(A2$^\prime$)--(A3$^\prime$) introduced below.
\begin{description}
\item[(A1)] \emph{Stationary or piecewise-stationary cells.} Conditional on the
chosen arm $a$ and bucket $k$, and averaging over the exogenous task draw
within that bucket, observations are independent across pulls and the joint
distribution of $(q_t,c_t,f_t)$ is identical within each phase.
\item[(A2)] \emph{Sub-Gaussian noise.} The utility $\ut$ given cell $(a,k)$ is
$\sigma_u$-sub-Gaussian around mean $\mu_u^{a,k}$; the realized cost $c_t$ is
$\sigma_{c,\mathrm{true}}$-sub-Gaussian around mean $\mu_c^{a,k}$. The policy
uses fixed proxy variances $\sigma_q^2,\sigma_c^2$, possibly different from the
true variances.
\item[(A3)] \emph{Bounded means.} $|\mu_u^{a,k}|\leq1$ and
$\mu_c^{a,k}\in[0,c_{\max}]$ for every $(a,k)$.
\end{description}

The stationary regret reduction further uses two restrictions that are not
needed by the boundedness, consistency, or adaptation claims:
\begin{description}
\item[(A2$^\prime$)] \emph{Well-specified product model.} The policy's
proxy likelihoods are the true likelihoods: utility and cost observations are
conditionally independent Gaussians around $(\mu_u^{a,k},\mu_c^{a,k})$ with
variances $\sigma_q^2=\sigma_u^2$ and
$\sigma_c^2=\sigma_{c,\mathrm{true}}^2$, and the prior factorizes as the
policy's independent Gaussian Q/C priors.
\item[(A3$^\prime$)] \emph{Clean score.} The regret reduction is stated for the
unclipped score
$(1-\lamn)\hatu{a}{k}-\lamn\hatc{a}{k}/c_{\max}$, or equivalently for
trajectories on which the implementation's cost-sample clip
\ifdefined\technicalreportonly
in the \padct{} score is inactive.
\else
in the \padct{} sampled reward (Eq.~\eqref{eq:sampled-pa-reward} of the main paper) is
inactive.
\fi
\end{description}

The posterior arguments use the Normal--Normal mean identities
\begin{equation}
\label{eq:posterior-mean-compact}
\thetaq{a}{k}
= \frac{\kappa_q\mu_{0,q}+S^{a,k}_t}{\kappa_q+n^{a,k}_t},
\qquad
\thetac{a}{k}
= \frac{\kappa_c\bar c_a+{S_c}^{a,k}_t}{\kappa_c+n^{a,k}_t},
\end{equation}
where $\kappa_q=\sigma_q^2/\sigma_{0,q}^2$ and
$\kappa_c=\sigma_c^2/\sigma_{0,c}^2$.
We write a single $n^{a,k}_t$ because each selected call supplies one utility
and one cost observation. If quality and cost discounts are split, the same
formulas apply to the two posteriors with their own effective counts.

\subsection{Proposition~\ref{prop:reward-bounded}: reward boundedness}
\label{app:proof-bounded}

\begin{proposition}[Reward boundedness]
\label{prop:reward-bounded}
For every round $t$ with $0\le\nu\le1$, the payment-aware reward satisfies
$\rt\in[-1,+1]$.
\end{proposition}

\begin{proof}
Since $q_t\in[0,1]$, $f_t\in\{0,1\}$, and $0\le\nu\le1$, we have
$\ut=q_t$ when $f_t=0$ and $\ut=q_t-\nu\in[-1,1]$ when $f_t=1$.
Thus $|\ut|\leq1$; also $\ctilde\in[0,1]$ and
$\lamn\in(0,1)$. Hence
$|\rt|\leq(1-\lamn)|\ut|+\lamn|\ctilde|\leq(1-\lamn)+\lamn=1$.
\end{proof}

\subsection{Proposition~\ref{prop:consistency}: posterior consistency}
\label{app:proof-consistency}

\begin{proposition}[Posterior consistency, stationary regime]
\label{prop:consistency}
Fix $(a,k)$ and assume the utility distribution at $(a,k)$ is stationary with
finite variance and mean $\mu_u^{a,k}$. With no discount ($\gamma=1$), the
Q-posterior mean $\thetaq{a}{k}$ converges almost surely to $\mu_u^{a,k}$ as the
cell's pull count $n^{a,k}\to\infty$. The same holds for the C-posterior and
$\mu_c^{a,k}$.
\end{proposition}

\begin{proof}
Let $\{\ut^{(i)}\}_{i=1}^n$ be the utility observations in cell $(a,k)$ after
$n$ pulls. By (A1), observations within this stationary cell are i.i.d.; the
strong law of large numbers therefore gives
$S^{a,k}_t/n^{a,k}_t=n^{-1}\sum_i\ut^{(i)}\to\mu_u^{a,k}$ almost surely.
Dividing the first expression in Eq.~\eqref{eq:posterior-mean-compact} by
$n^{a,k}_t$ in numerator and denominator gives
$\thetaq{a}{k}\to\mu_u^{a,k}$ because $\kappa_q/n^{a,k}_t\to0$. The cost
posterior is identical with $\bar c_a,{S_c}^{a,k}_t,\kappa_c$ in place of
$\mu_{0,q},S^{a,k}_t,\kappa_q$.
\end{proof}

\subsection{Lemma: discounted effective sample size}
\label{app:lemma-ess}

\begin{lemma}[Effective sample size]
\label{lem:ess}
For any $\gamma\in(0,1)$ and any sequence of pulls of cell $(a,k)$,
\[
n^{a,k}_t \leq \frac{1-\gamma^{t+1}}{1-\gamma}\leq\frac{1}{1-\gamma}.
\]
The first bound is tight when $(a,k)$ is pulled at every round.
\end{lemma}

\begin{proof}
If $\mathcal{T}^{a,k}_t\subseteq\{0,\ldots,t\}$ is the set of pull rounds for
cell $(a,k)$, the round-based discount rule gives
$n^{a,k}_t=\sum_{s\in\mathcal{T}^{a,k}_t}\gamma^{t-s}
\leq\sum_{j=0}^t\gamma^j=(1-\gamma^{t+1})/(1-\gamma)$.
\end{proof}

\subsection{Stationary analysis}
\label{app:proof-stationary}

We now turn from posterior properties to the stationary decision rule. Under
the well-specified product-Gaussian model and clean-score restriction,
\padct{} induces a two-dimensional linear Gaussian Thompson-sampling problem.

\begin{theorem}[Clean-score regret reduction to Gaussian Thompson sampling]
\label{thm:stationary-regret}
Under (A1), (A2$^\prime$), (A3$^\prime$), and $\gamma=1$, define the
\emph{policy-state PA-regret} of clean-score \padct{} against the wallet-coupled
oracle as
\begin{equation}
\label{eq:policy-state-regret}
\begin{aligned}
g_t(a) &:=
(1-\lamn)\,\mu_u^{a,k_t}
- \lamn\,\mu_c^{a,k_t}/c_{\max},\\
\Reg_T &:= \sum_{t=0}^{T-1}\big(g_t(a^\star_t)-g_t(a_t)\big),
\end{aligned}
\end{equation}
where $a^\star_t\in\arg\max_{a\in\mathcal{A}_t}g_t(a)$ at the policy's wallet
pressure $\lamn$, and $a_t$ is \padct{}'s chosen arm. Conditional on the
bucket-visit sequence, clean-score \padct{} reduces to per-bucket Gaussian
Thompson sampling with a time-varying linear context. For a bucket visited
$T_k$ times,
\begin{equation}
\label{eq:cell-regret}
\mathbb{E}\!\left[\Reg^{\mathrm{cell}\,k}_{T_k}\right]
\leq C_0\sqrt{K\,T_k\,\log T_k},
\end{equation}
for a constant $C_0$ depending on the prior and likelihood variances and on
$c_{\max}$. Aggregating across $K_b$ buckets with $\sum_kT_k=T$ gives
\begin{equation}
\label{eq:stationary-regret}
\mathbb{E}\!\left[\Reg_T\right]\leq C\sqrt{K\,K_b\,T\log T}.
\end{equation}
\end{theorem}

\begin{proof}
For a visit to bucket $k$, write $w_t=\lamn$ and define the observed feature
\[
\phi_t=(1-w_t,\,-w_t/c_{\max}),\qquad
\theta^{a,k}=(\mu_u^{a,k},\,\mu_c^{a,k}).
\]
Then $g_t(a)=\phi_t^\top\theta^{a,k}$. Under (A2$^\prime$), the independent
Q/C Thompson samples used by \padct{} imply
\[
\begin{aligned}
&(1-w_t)\hatu{a}{k}-w_t\hatc{a}{k}/c_{\max}\\
&\quad\sim
\mathcal{N}\!\left(
\phi_t^\top(\thetaq{a}{k},\thetac{a}{k}),
(1-w_t)^2\Sigmaq{a}{k}
+w_t^2\Sigmac{a}{k}/c_{\max}^2
\right),
\end{aligned}
\]
which is exactly the product-Gaussian posterior sample of
$\phi_t^\top\theta^{a,k}$. This identifies the sampled score with a Gaussian
posterior draw for the induced linear payoff.

The feature $\phi_t$, bucket $k_t$, and affordable set $\mathcal{A}_t$ are
observed before sampling. Bucket visits are determined by the exogenous task
\ifdefined\technicalreportonly
stream of the benchmark setting,
\else
stream in \S\ref{sec:problem},
\fi
so conditioning on a bucket subsequence does not introduce policy-dependent
context selection; the wallet pressure and affordable set are predictable state
shared by the policy-state comparator in Eq.~\eqref{eq:policy-state-regret}.
Therefore the per-round gap is the standard contextual Thompson-sampling regret
for at most $K$ arms.

Consequently, on each bucket subsequence the clean-score policy is
mathematically equivalent to Gaussian Thompson sampling on the induced linear
payoff. Applying the Bayesian regret bound for $K$-armed Gaussian-prior
Thompson sampling \citep[Ch.~36]{lattimore2020bandit} gives
Eq.~\eqref{eq:cell-regret}; the underlying confidence-bound decomposition of
posterior sampling \citep[Prop.~1]{russo2014learning} carries over because
$\phi_t$ and $\mathcal{A}_t$ are $H_t$-measurable, so the posterior-matching
identity holds conditionally on $H_t$. Summing over buckets and using
Cauchy--Schwarz,
\[
\sum_k\sqrt{T_k\log T_k}\leq\sqrt{K_bT\log T}.
\]
This gives Eq.~\eqref{eq:stationary-regret}.
\end{proof}

\subsection{Nonstationary analysis}
\label{app:proof-adaptation}

We next isolate how exponential discounting removes stale evidence after an
abrupt change. The result permits any fixed or reward-noise-independent revisit
schedule, followed by a continuous-pull corollary with an explicit recovery
horizon.

\begin{theorem}[Adaptation rate under an abrupt change, fixed schedule]
\label{thm:adaptation}
Fix $\gamma\in(0,1)$ and cell $(a,k)$. Suppose the true utility mean shifts
between rounds $t^\star-1$ and $t^\star$, with pre-shock mean
$\mu_{\mathrm{old}}$, post-shock mean $\mu_{\mathrm{new}}$, and
$\Delta=|\mu_{\mathrm{old}}-\mu_{\mathrm{new}}|$. Consider any fixed, or
reward-noise-independent, pull schedule for this cell covering pre- and
post-shock pulls, with post-shock pull rounds
$\{s_1<\cdots<s_m\}\subseteq[t^\star,t]$. Define
\[
\begin{aligned}
A &:= \gamma^{t-t^\star+1}n^{a,k}_{t^\star-1},\\
B &:= \sum_{i=1}^m\gamma^{t-s_i},\\
\kappa_q &:= \sigma_q^2/\sigma_{0,q}^2 .
\end{aligned}
\]
Here $A$ is the discounted mass of stale pre-shock observations and $B$ is the
mass of new post-shock observations. Then the Bayesian Q-posterior mean
satisfies
\begin{equation}
\label{eq:adaptation}
\left|\,\mathbb{E}[\theta^{a,k}_{q,t}]-\mu_{\mathrm{new}}\,\right|
\leq
\frac{A\Delta+\kappa_q|\mu_{0,q}-\mu_{\mathrm{new}}|}
{\kappa_q+A+B},
\end{equation}
where the expectation is over observation noise for the fixed schedule.
\end{theorem}

\begin{remark}[Two regimes]
\label{rem:adapt-regimes}
The bound~\eqref{eq:adaptation} separates a prior term from a stale-data term.
When $\kappa_q\gg A+B$, the prior controls the bias. When
$\kappa_q\ll A+B$, the leading term is $\Delta A/(A+B)$, which is reduced by
discounting and by new post-shock pulls.
\end{remark}

\begin{corollary}[Recovery horizon, data-dominated regime]
\label{cor:recovery}
Suppose $(a,k)$ is pulled at every post-shock round, so
$\ell:=t-t^\star+1=m$, and assume the data-dominated regime
$\kappa_q\ll A+B$. Then
\[
\left|\,\mathbb{E}[\theta^{a,k}_{q,t}]-\mu_{\mathrm{new}}\,\right|
\lesssim
\Delta\gamma^\ell+O\!\big(\kappa_q/(A+B)\big),
\]
and the leading stale-data term is at most $\delta$ once
\[
\ell\geq\frac{\log(\Delta/\delta)}{\log(1/\gamma)}
=O\!\big(\log(\Delta/\delta)/(1-\gamma)\big).
\]
\end{corollary}

\begin{proof}[Proof of Theorem~\ref{thm:adaptation}]
Because the entire pull schedule is fixed or independent of the reward noise,
conditioning on the pull times does not change the mean of the included
observations; in particular, the pre-shock mass $n^{a,k}_{t^\star-1}$ is
determined by the schedule. From Eq.~\eqref{eq:posterior-mean-compact},
\[
\theta^{a,k}_{q,t}=\frac{\kappa_q\mu_{0,q}+S^{a,k}_t}
{\kappa_q+n^{a,k}_t}.
\]
Splitting the discounted mass into pre- and post-shock pulls gives
$n^{a,k}_t=A+B$ and
$\mathbb{E}[S^{a,k}_t]=A\mu_{\mathrm{old}}+B\mu_{\mathrm{new}}$. Hence
\[
\mathbb{E}[\theta^{a,k}_{q,t}]-\mu_{\mathrm{new}}
=
\frac{\kappa_q(\mu_{0,q}-\mu_{\mathrm{new}})
+A(\mu_{\mathrm{old}}-\mu_{\mathrm{new}})}
{\kappa_q+A+B}.
\]
Taking absolute values and using
$|\mu_{\mathrm{old}}-\mu_{\mathrm{new}}|=\Delta$ gives
Eq.~\eqref{eq:adaptation}.
\end{proof}

\begin{proof}[Proof of Corollary~\ref{cor:recovery}]
If the cell is pulled at every post-shock round, then
$B=\sum_{j=0}^{\ell-1}\gamma^j=(1-\gamma^\ell)/(1-\gamma)$ and
$A=\gamma^\ell n^{a,k}_{t^\star-1}$. By Lemma~\ref{lem:ess},
$n^{a,k}_{t^\star-1}\leq1/(1-\gamma)$, so
\[
\frac{A}{A+B}
\leq
\frac{\gamma^\ell/(1-\gamma)}
{\gamma^\ell/(1-\gamma)+(1-\gamma^\ell)/(1-\gamma)}
=\gamma^\ell .
\]
Substituting this into Eq.~\eqref{eq:adaptation} yields the displayed bound;
solving $\Delta\gamma^\ell\leq\delta$ gives the stated horizon.
\end{proof}

\begin{remark}[Cost-posterior analogue]
The same calculation applies to realized-cost shocks: replace
$(\mu_{\mathrm{old}},\mu_{\mathrm{new}},\mu_{0,q},\kappa_q)$ with
$(\mu^c_{\mathrm{old}},\mu^c_{\mathrm{new}},\bar c_a,\kappa_c)$ and use the
C-posterior statistic under the same locked discount $\gamma$. This is the
algebra behind S3-style price changes once the changed arm is pulled again;
we state the utility version to avoid duplicating notation.
\end{remark}

\subsection{Scope of the formal analysis}
\label{app:proof-limitations}

The results above are component-level formal support, not a complete regret
theorem for the full replay system. The stationary regret reduction assumes a
well-specified product-Gaussian model, $\gamma=1$, and the clean score with
inactive cost-sample clipping; the implemented clipped scorer coincides with
this policy on trajectories where the clip is inactive. The adaptation result
is conditional on a fixed, noise-independent pull schedule: it explains how discounted
posteriors forget stale evidence after a changed arm is revisited, but it does
not prove an unconditional Thompson-sampling revisit rate or a cumulative
piecewise-stationary regret bound. Finally, the theorem's comparator uses the
policy's own wallet pressure and affordable set, whereas the empirical
\emph{True Oracle} has its own wallet trajectory.

\section{Detailed Ablation Study}
\label{appendix:ablations}

This appendix reports the full metrics behind the component-role analysis in
\ifdefined\technicalreportonly
the main report.
\else
\S\ref{sec:eval-ablations} of the main paper.
\fi
The four primary variants each modify one component of Full \padct{}:
$-P$ removes payment-aware ranking and scores providers by sampled utility
alone; $-D$ removes discounting ($\gamma=1$); $-C$ removes contextual
bucketing by collapsing the four task-type buckets into one; and
$-\mathrm{TS}$ replaces Thompson samples with posterior means.
The additional price-shock diagnostic $-C_{\mathrm{post}}$ retains wallet
pressure but pins the cost estimate to the specification cost $\bar c_a$ and
is evaluated only in S3.
PA-gap/$T$ and full-horizon mean quality are reported as endpoint metrics,
whereas ROI and AdaptT serve as diagnostics. AdaptT is the number of rounds
after the shock until the first trailing-$200$ ROI window reaches $95\%$ of
pre-shock ROI in S2 or $110\%$ in S3. We average AdaptT over seeds that cross
the threshold and report the corresponding hit count out of~$30$.

\begin{table*}[t]
\centering
\small
\setlength{\tabcolsep}{4pt}
\begin{tabular*}{\textwidth}{@{\extracolsep{\fill}}llcccc@{}}
\toprule
Scenario & Variant & \shortstack{PA-gap/$T$\\$\downarrow$} & ROI $\uparrow$ & \shortstack{Mean quality\\$\uparrow$} & \shortstack{AdaptT\\(rounds) $\downarrow$} \\
\midrule
S1 & Full & $0.133{\pm}0.004$ & $378$ & $0.797$ & \texttt{n/a} \\
S1 & $-P$ & $0.855{\pm}0.058$ & $111$ & $0.555^{\ddagger}$ & \texttt{n/a} \\
S1 & $-D$ & $0.109{\pm}0.008$ & $412$ & $0.810$ & \texttt{n/a} \\
S1 & $-C$ & $0.116{\pm}0.003$ & $390$ & $0.812$ & \texttt{n/a} \\
S1 & $-\mathrm{TS}$ & $0.130{\pm}0.046$ & $535$ & $0.740$ & \texttt{n/a} \\
\midrule
S2 & Full & $0.166{\pm}0.006$ & $357$ & $0.761$ & $1{,}467$ (30/30) \\
S2 & $-P$ & $0.878{\pm}0.050$ & $103$ & $0.517^{\ddagger}$ & $1{,}611$ (14/30) \\
S2 & $-D$ & $0.170{\pm}0.012$ & $399$ & $0.745$ & $2{,}249$ (30/30) \\
S2 & $-C$ & $0.158{\pm}0.006$ & $368$ & $0.761$ & $1{,}176$ (30/30) \\
S2 & $-\mathrm{TS}$ & $0.175{\pm}0.048$ & $537$ & $0.687$ & $1{,}125$ (30/30) \\
\midrule
S3 & Full & $0.121{\pm}0.004$ & $429$ & $0.831$ & $200$ (30/30) \\
S3 & $-P$ & $0.259{\pm}0.036$ & $351$ & $0.840$ & $200$ (30/30) \\
S3 & $-D$ & $0.114{\pm}0.009$ & $426$ & $0.838$ & $279$ (30/30) \\
S3 & $-C$ & $0.109{\pm}0.003$ & $425$ & $0.848$ & $398$ (29/30) \\
S3 & $-\mathrm{TS}$ & $0.145{\pm}0.025$ & $552$ & $0.742$ & $2{,}232$ (24/30) \\
S3 & $-C_{\mathrm{post}}$ & $0.144{\pm}0.005$ & $423$ & $0.797$ &
$200$ (30/30) \\
\bottomrule
\end{tabular*}
\caption{Full component-ablation metrics on 402Pilot-Bench. PA-gap/$T$ is
mean$\pm$SD over $30$ paired seeds; ROI and quality entries are means.
AdaptT is the mean elapsed rounds over hit seeds (hits/30), is lower-bounded
by $200$, and is a diagnostic rather than an endpoint.
${}^{\ddagger}$ indicates that all $30$ seeds exhausted the wallet before
$T$; \texttt{n/a} denotes a scenario without a shock.}
\label{tab:ablations-full}
\end{table*}

Two reading notes accompany Table~\ref{tab:ablations-full}. First, the $-P$
row in S2 reports AdaptT over only the $14/30$ seeds that reached the ROI
threshold before exhausting the wallet. Its reported AdaptT is therefore
conditional on the successful subset and should not be compared directly
with the full-hit rows. Second, $-P$ and Contextual DS-TS in
Table~\ref{tab:main} implement the same wallet-pressure-free selection rule
through independent code paths and RNG streams. Their PA-gap/$T$ values differ
by at most $0.007$ across scenarios, providing an implementation-level
consistency check.

\section{Sensitivity and Robustness Checks}
\label{appendix:sensitivity}

This appendix reports the sensitivity checks behind the locked
\ifdefined\technicalreportonly
configuration used in the main evaluation.
\else
configuration in \S\ref{sec:eval-setup} of the main paper.
\fi
We first sweep the four exposed selection scalars one at a time on S1, use the
ablation panel for the cross-scenario discount check, and treat posterior
variances, $T$, and $B$ as fixed calibration constants rather than per-scenario
tuning knobs. We then examine comparator-specific sweeps, the SW-TS diagnostic,
and paired-seed statistical checks. Throughout, the goal is to verify that the
locked configuration is a reasonable calibration, not to search for a better
one.
\ifdefined\technicalreportonly\else
Comparator implementation details remain in
Appendix~\ref{sec:app-extended-baselines}.
\fi

\subsection{Targeted S1 sensitivity sweep}
\label{app:s1-sensitivity}

Each non-default row in Table~\ref{tab:s1-sensitivity} changes one scalar
while holding the others fixed, using $10$ seeds at $T=10{,}000$. The True
Oracle is recomputed for each cell with the same $\nu$ and $c_{\max}$ as the
tested policy, so PA-gap/$T$ remains comparable within the sweep.

\begin{table*}[t]
\centering
\small
\setlength{\tabcolsep}{5pt}
\begin{tabular*}{\textwidth}{@{\extracolsep{\fill}}lrrrr@{}}
\toprule
Setting & $\bar q_T$ & Budget \% & ROI & PA-gap/$T$ \\
\midrule
\textsc{default}       & $0.796{\pm}0.001$ & $42.2{\pm}0.3$ & $377.5{\pm}2.1$ & $0.1328{\pm}0.0013$ \\
\midrule
$\gamma=0.99$      & $0.753{\pm}0.001$ & $49.8{\pm}0.2$ & $302.4{\pm}1.3$ & $0.2087{\pm}0.0013$ \\
$\gamma=0.9995$    & $0.801{\pm}0.002$ & $40.8{\pm}0.4$ & $392.6{\pm}3.3$ & $0.1218{\pm}0.0016$ \\
$\nu=0.25$         & $0.795{\pm}0.001$ & $42.1{\pm}0.4$ & $377.6{\pm}3.0$ & $0.1329{\pm}0.0018$ \\
$\nu=1.0$          & $0.797{\pm}0.001$ & $42.1{\pm}0.3$ & $378.7{\pm}2.3$ & $0.1324{\pm}0.0014$ \\
$\alpha=1.0$       & $0.788{\pm}0.002$ & $36.7{\pm}0.3$ & $430.3{\pm}3.1$ & $0.1045{\pm}0.0011$ \\
$\alpha=4.0$       & $0.802{\pm}0.001$ & $54.7{\pm}0.4$ & $293.5{\pm}2.0$ & $0.1781{\pm}0.0016$ \\
$c_{\max}=0.005$   & $0.787{\pm}0.001$ & $41.9{\pm}0.4$ & $375.9{\pm}3.0$ & $0.1567{\pm}0.0028$ \\
$c_{\max}=0.02$    & $0.800{\pm}0.001$ & $50.5{\pm}0.3$ & $316.8{\pm}1.9$ & $0.1490{\pm}0.0021$ \\
\bottomrule
\end{tabular*}
\caption{One-at-a-time S1 sensitivity of \padct{}. Means $\pm$ SE over $10$
seeds; each non-default row perturbs a single scalar.}
\label{tab:s1-sensitivity}
\end{table*}

The sweep supports two narrow claims. First, the locked policy is not brittle:
no cell bankrupts the wallet, $\bar q_T$ stays within $[0.753,0.802]$, and
PA-gap/$T$ remains in the same order of magnitude. Second, the visible levers
are the intended ones: shorter memory worsens stationary performance, and
larger $\alpha$ spends more aggressively; $\nu$ and $c_{\max}$ have smaller
effects in the tested range.

\subsection{\texorpdfstring{Discount factor $\gamma$}{Discount factor gamma}}
\label{app:sens-gamma}

The cross-scenario check uses the $-D$ rows of
Table~\ref{tab:ablations-full}: full \padct{} sets $\gamma=0.999$, while $-D$
sets $\gamma=1$. On S1, no discount is slightly better on the stationary
endpoints, but discounting shortens AdaptT in both shock scenarios
($1{,}467$ vs.\ $2{,}249$ rounds in S2; $200$ vs.\ $279$ in S3). This is the
empirical counterpart of Cor.~\ref{cor:recovery}, where stale evidence decays
as $\gamma^\ell$ after a changed arm is revisited. The default
$\gamma=0.999$ corresponds to an effective horizon of
$1/(1-\gamma)=10^3$ rounds, chosen to retain S1 evidence while still reacting
inside the S2 outage window.

\subsection{Other locked constants}
\label{app:sens-other}

The failure penalty $\nu=0.5$ gives billed timeouts a cost beyond their zero
quality score; the sweep shows that halving or doubling it barely moves the
metrics. Likewise, the pressure schedule uses $\alpha=2$ and $\lambda_0=1$, so
$\lamn=0.5$ when spending is on plan and rises smoothly under sustained
overshoot; the sweep confirms the expected spend trade-off between $\alpha=1$
and $\alpha=4$. The normalizer $c_{\max}=\$0.01$ is the highest listed
provider price, and the half/double sweep does not change the qualitative
conclusion.

The posterior prior and likelihood variances are intentionally not swept: the
Q-posterior uses a broad mid-range prior with moderate proxy noise, while the
C-posterior uses a tight prior and likelihood because realized cost is nearly
deterministic except at the S3 price shock. We also do not sweep $T$ or $B$:
the released benchmark fixes $T=10{,}000$, and $B=\$50$ is calibrated so that
a premium-only policy hits the budget cliff in the second half of the run.

\subsection{Comparator sensitivity sweeps}
\label{app:new-baseline-sensitivity}

Table~\ref{tab:baseline-sweeps} sweeps the PM-Greedy threshold
$\theta_{\mathrm{PM}}$ and the LinCBwK-Adapt exploration coefficient $\beta$;
the main table uses the locked settings $\theta_{\mathrm{PM}}=0.7$ and
$\beta=1.0$. Reducing either $\theta_{\mathrm{PM}}$ or $\beta$ improves the
corresponding comparator's PA-gap/$T$ in every scenario. For reference,
\padct{}'s mean PA-gap/$T$ values in the main results
\ifdefined\technicalreportonly
are
\else
(Table~\ref{tab:main} of the main paper) are
\fi
$0.133$, $0.166$, and $0.121$ in S1--S3. Across these comparator sweeps,
\padct{} retains the lowest mean PA-gap/$T$ among the tested settings, with
the margin largest in S1/S2 and narrower in S3.

\begin{table}[t]
\centering
\small
\setlength{\tabcolsep}{2pt}
\begin{tabular*}{\columnwidth}{@{\extracolsep{\fill}}llrrr@{}}
\toprule
Setting & Scenario & $\bar q_T$ & Budget \% & PA-gap/$T$ \\
\midrule
\multicolumn{5}{@{}l}{\emph{PM-Greedy quality threshold $\theta_{\mathrm{PM}}$}} \\
$\theta_{\mathrm{PM}}=0.6$ & S1 & $0.776$ & $55.3$ & $0.244$ \\
$\theta_{\mathrm{PM}}=0.6$ & S2 & $0.733$ & $80.6$ & $0.364$ \\
$\theta_{\mathrm{PM}}=0.6$ & S3 & $0.793$ & $35.4$ & $0.152$ \\
$\theta_{\mathrm{PM}}=0.7$ & S1 & $0.732$ & $89.5$ & $0.542$ \\
$\theta_{\mathrm{PM}}=0.7$ & S2 & $0.627$ & $98.7$ & $0.660$ \\
$\theta_{\mathrm{PM}}=0.7$ & S3 & $0.836$ & $43.9$ & $0.198$ \\
$\theta_{\mathrm{PM}}=0.8$ & S1 & $0.667$ & $94.6$ & $0.684$ \\
$\theta_{\mathrm{PM}}=0.8$ & S2 & $0.581$ & $99.1$ & $0.763$ \\
$\theta_{\mathrm{PM}}=0.8$ & S3 & $0.836$ & $46.8$ & $0.243$ \\
\addlinespace[2pt]
\multicolumn{5}{@{}l}{\emph{LinCBwK-Adapt exploration coefficient $\beta$}} \\
$\beta=0.5$ & S1 & $0.803$ & $70.0$ & $0.312$ \\
$\beta=0.5$ & S2 & $0.773$ & $83.5$ & $0.386$ \\
$\beta=0.5$ & S3 & $0.806$ & $37.8$ & $0.164$ \\
$\beta=1.0$ & S1 & $0.824$ & $93.0$ & $0.457$ \\
$\beta=1.0$ & S2 & $0.796$ & $99.9$ & $0.515$ \\
$\beta=1.0$ & S3 & $0.836$ & $44.5$ & $0.195$ \\
$\beta=2.0$ & S1 & $0.832$ & $100.0$ & $0.515$ \\
$\beta=2.0$ & S2 & $0.800$ & $100.0$ & $0.533$ \\
$\beta=2.0$ & S3 & $0.854$ & $45.7$ & $0.197$ \\
\bottomrule
\end{tabular*}
\caption{One-at-a-time sensitivity sweeps for PM-Greedy and LinCBwK-Adapt.
Means over $30$ paired seeds.}
\label{tab:baseline-sweeps}
\end{table}

\subsection{SW-TS diagnostic}
\label{app:sw-ts}

SW-TS adapts sliding-window Thompson sampling~\citep{trovo2020sliding} to our
admissible loop. It maintains a
non-contextual Gaussian reward posterior per arm over utility observations
from the last $W=1000$ global rounds and samples one utility draw per
affordable arm. The window is indexed by global rounds: observations are
timestamped by global round index and pruned once older than $W$ rounds,
whether or not
the arm was recently pulled; if an arm's active window is empty, the posterior
reverts to the prior (mean $0.5$, variance $1.0$, likelihood variance
$0.09$). SW-TS does not learn cost and has no wallet-pressure term; the
affordable-set mask is its only budget signal.

Table~\ref{tab:sw-ts-results} reports the SW-TS endpoint metrics. For
bankrupt seeds, aggregate metrics are computed over the full horizon, with
post-bankruptcy rounds contributing zero quality and no further spend.
Without wallet pressure, SW-TS repeatedly selects high-quality expensive
arms; in S1/S2 this drains the wallet for all $30$ seeds.

\begin{table}[t]
\centering
\small
\setlength{\tabcolsep}{2pt}
\begin{tabular*}{\columnwidth}{@{\extracolsep{\fill}}lrrrrr@{}}
\toprule
Scenario & $\bar q_T$ & Budget \% & ROI & PA-gap/$T$ & Bankrupt \\
\midrule
S1 & $0.483$ & $100.0$ & $97$  & $0.981$ & $30/30$ \\
S2 & $0.466$ & $100.0$ & $93$  & $0.986$ & $30/30$ \\
S3 & $0.855$ & $52.3$  & $328$ & $0.324$ & $0/30$ \\
\bottomrule
\end{tabular*}
\caption{SW-TS appendix results. Means over $30$ paired seeds; bankruptcy is
the number of seeds that exhausted the wallet before $T=10{,}000$.}
\label{tab:sw-ts-results}
\end{table}

\subsection{Paired-seed statistical comparisons}
\label{sec:app-paired-tests}

For each scenario--metric pair, we compare \padct{} with each baseline using
the $30$ within-seed differences. Reported $p$-values use two-sided paired
normal-approximation tests. As a robustness check against skew from early
budget exhaustion, we also compute $95\%$ paired-bootstrap confidence
intervals with $10{,}000$ resamples; these intervals support the same
conclusions.

\padct{}'s S3 improvements over Always-P-mid on quality, ROI, and PA-gap/$T$
are significant (all $p<10^{-20}$), as is its S2 quality gain over
Always-P-mid ($p=1.67\times10^{-3}$). Its S1 PA-gap/$T$ disadvantage against
Always-P-mid is also significant, consistent with exploration cost. In S2,
\padct{}'s PA-gap/$T$ is not distinguishable from
Always-P-cheap ($p=0.17$) or Contextual BTS ($p=0.86$). Against PM-Greedy,
LinCBwK-Adapt, and Contextual DS-TS, its PA-gap/$T$ advantage is significant
in all three scenarios.

\section{Allocation Dynamics under Market Shocks}
\label{appendix:allocation-dynamics}

Aggregate metrics summarize the overall quality--cost trade-off but do not
show how the policy reallocates over time.
Figure~\ref{fig:allocation-dynamics} reports \padct{}'s provider-allocation
trajectories across the three market scenarios. S1 serves as a stationary
control; S2 shows withdrawal from and return to P-mid around the outage; and
S3 shows migration toward discounted P-premium. These trajectories complement
the AdaptT diagnostic in Table~\ref{tab:ablations-full} by showing when and
how provider allocation changes.

\begin{figure}[!t]
\centering
\includegraphics[width=0.82\columnwidth]{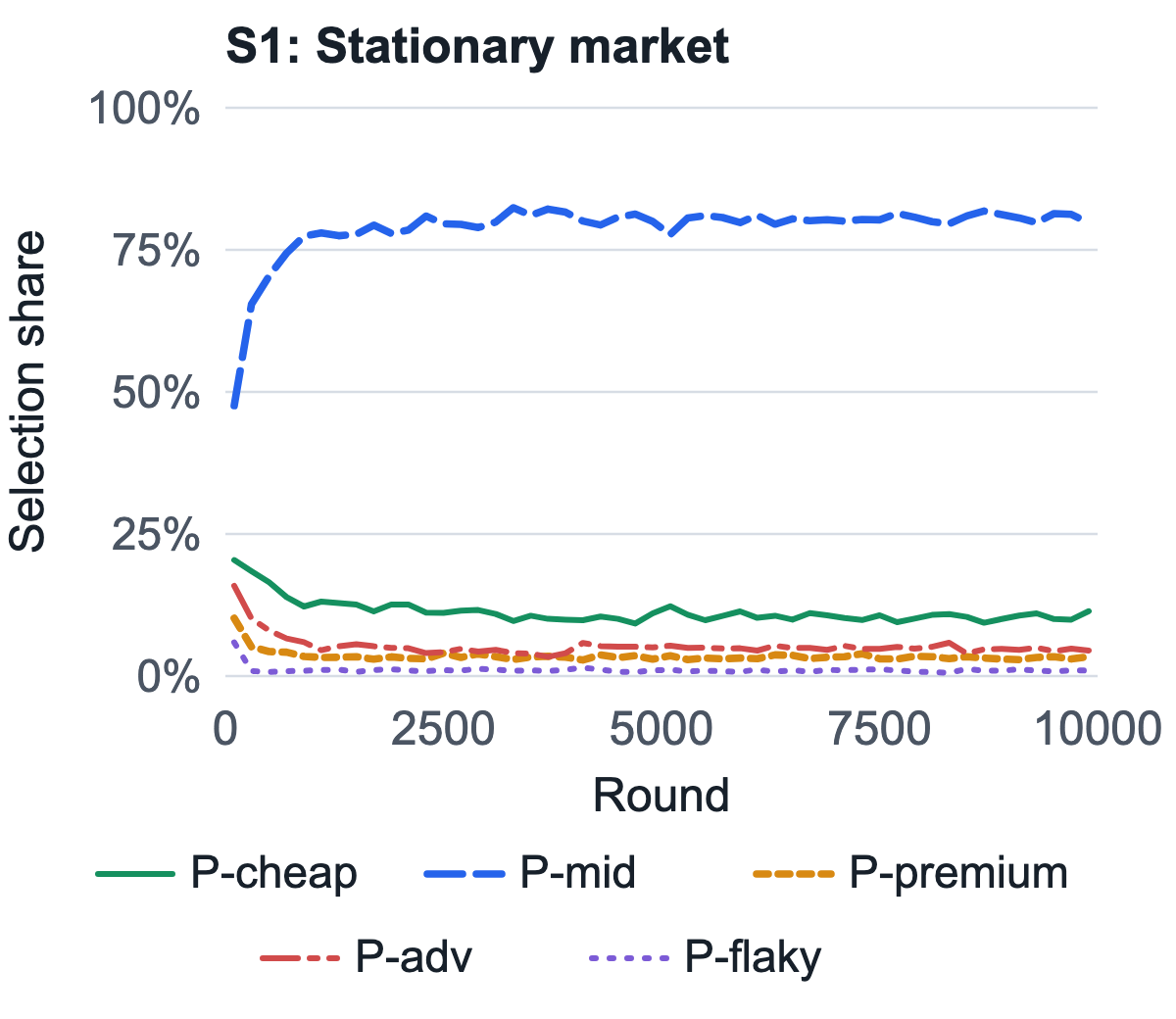}
\vspace{0.35em}
\includegraphics[width=0.82\columnwidth]{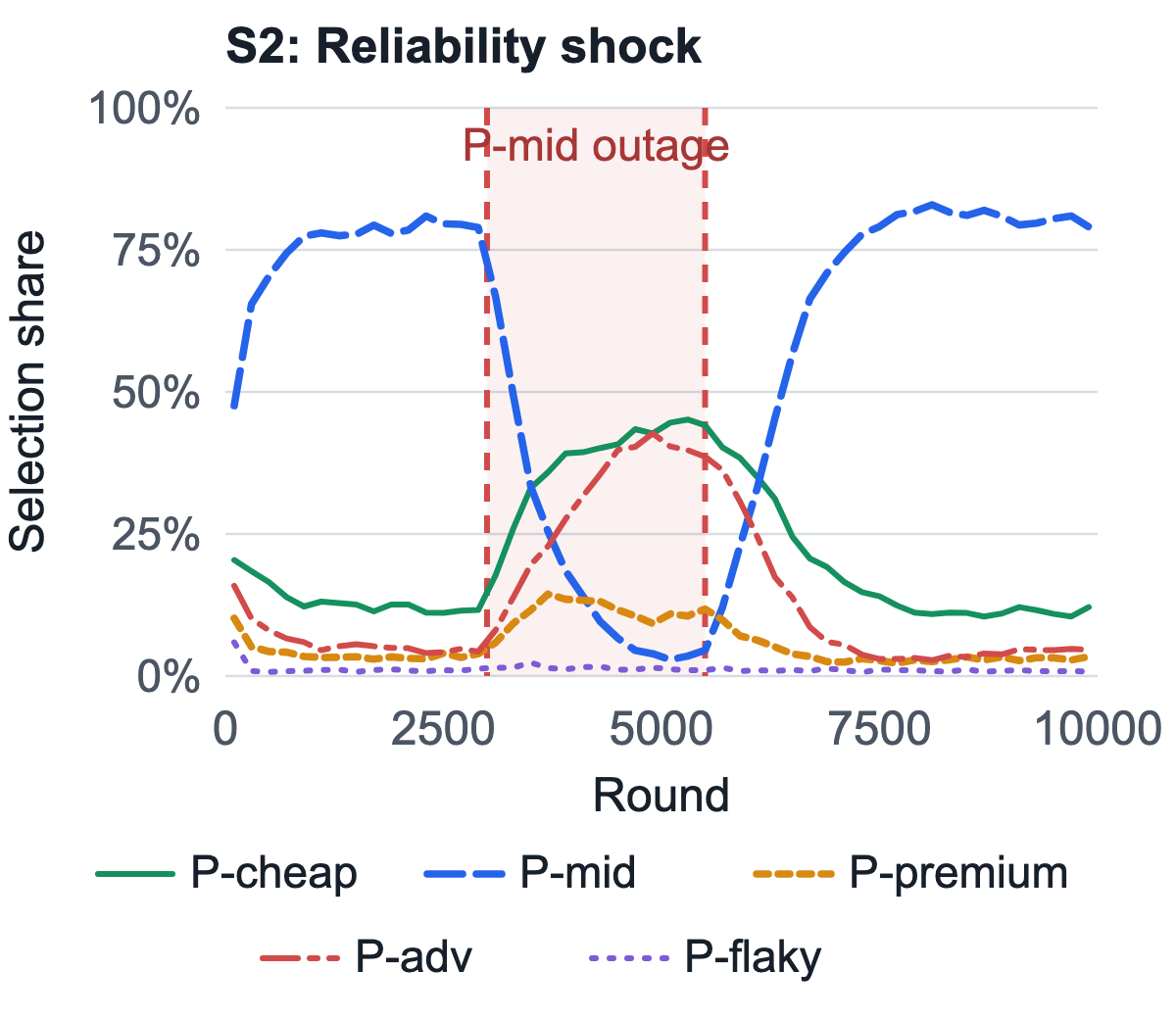}
\vspace{0.35em}
\includegraphics[width=0.82\columnwidth]{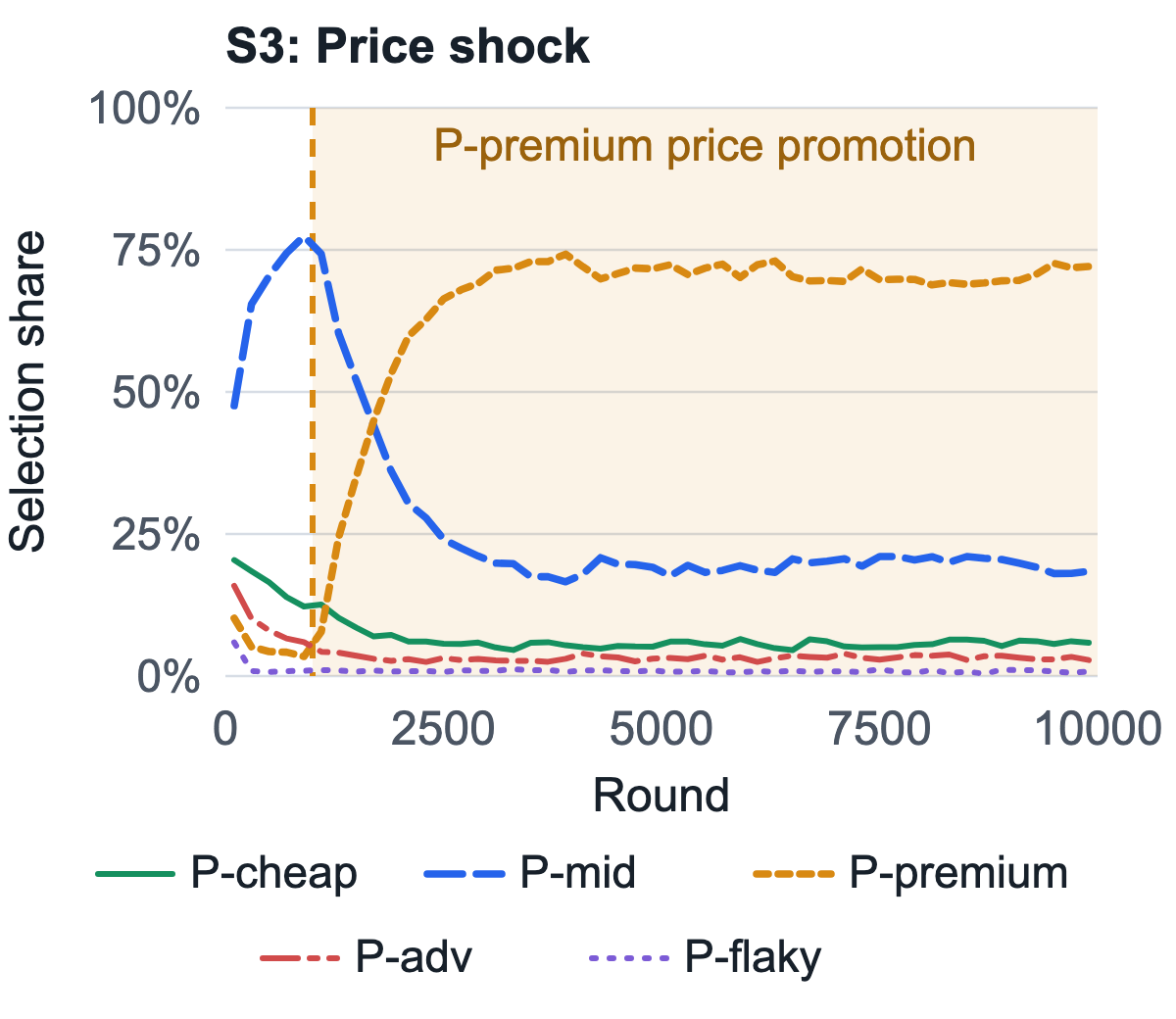}
\caption{\padct{} provider-allocation shares over $T=10{,}000$ rounds in,
from top to bottom, the stationary market (S1), reliability-shock scenario
(S2), and price-shock scenario (S3). In S1, P-mid remains dominant after
initial exploration. In S2, the policy shifts away from P-mid during its
outage and returns after recovery. In S3, it reallocates toward discounted
P-premium.}
\Description{Three vertically stacked line charts of PA-DCT provider-selection
shares in S1, S2, and S3. They show stable P-mid allocation, withdrawal from
and return to P-mid around its outage, and migration toward P-premium after
its price promotion, respectively.}
\label{fig:allocation-dynamics}
\end{figure}

\section{Extended Comparators}
\label{sec:app-extended-baselines}

\subsection{Comparator Implementations}

This section describes how each comparator is implemented in
\emph{402Pilot-Bench} under the common frozen-replay evaluation protocol.
Where necessary, standard algorithms are minimally adapted to operate under
replay while preserving their original decision principles. All non-oracle
implementations receive the same task context and hard affordable-set mask.
After each decision, the loop reveals only the chosen provider's
failure-penalized utility and realized receipt cost; each implementation uses
only the signals specified below. No comparator observes counterfactual cached
outcomes, and no learning comparator uses a forced round-robin warm start.

\paragraph{Random and fixed-arm policies.}
We implement Random by sampling uniformly from the affordable set.
Always-P-cheap, Always-P-mid, and Always-P-premium return their named provider
when it is affordable and otherwise return the first provider in the runtime's
fixed affordable-set ordering. These policies ignore context and replay
feedback.

\paragraph{BudgetRule.}
We implement BudgetRule as a deterministic wallet-threshold policy. It selects
P-premium when the remaining-budget fraction is above \(0.50\), P-mid when it
lies in \((0.20,0.50]\), and P-cheap otherwise. If the target tier
is unavailable, it selects the affordable provider with the highest catalog
price. It ignores context and replay feedback.

\paragraph{Contextual DS-TS.}
In our implementation, Contextual DS-TS adapts discounted Thompson
sampling~\citep{qi2023dsts} to four task buckets and continuous utility
feedback. It maintains one Gaussian utility posterior per
provider--task-bucket cell, discounts every cell's sufficient statistics once
per replay round, and samples one utility value for each affordable provider
in the current bucket. Only the chosen cell is updated. Receipt cost and
wallet state do not enter the rule beyond the runtime's affordable-set mask.

\paragraph{Contextual BTS.}
In our replay implementation, Contextual BTS adapts budgeted Thompson
sampling~\citep{xia2015bts} to four task buckets and continuous utility and
cost feedback. It maintains stationary Gaussian utility and cost posteriors
for each provider--task-bucket cell, with each cost prior centered at that
provider's catalog price. For every affordable provider, it samples both
quantities and selects the largest sampled utility-to-cost ratio; the sampled
cost is floored at \(10^{-6}\) USDC. Only the chosen cell is updated from its
utility and receipt cost, and wallet state enters only through affordability.

\paragraph{PM-Greedy.}
\label{app:pm-greedy}
PM-Greedy is implemented as a price-metadata router inspired by
cheapest-sufficient routing~\citep{chen2023frugalgpt}. For each
provider--task-bucket cell, it keeps a per-pull window of the \(500\) most
recent observed utilities; a cell with no observations yet uses the
optimistic estimate \(1\). Among
affordable providers whose mean utility is at least \(0.7\), it selects the
one with the lowest scenario-aware posted price, breaking price ties by higher
estimated utility. If none clears the threshold, it selects the highest-utility
provider, breaking ties by lower posted price. Receipt cost is not used for
learning, and the affordable-set mask is the only wallet signal.

\paragraph{LinCBwK-Adapt.}
\label{app:lincbwk}
Our replay adaptation of LinCBwK~\citep{agrawal2016linearcbwk} is an
admissible empirical comparator rather than a reproduction of the full
theoretical algorithm. With the four-dimensional one-hot task context, its
linear reward and cost models reduce to stationary Gaussian/ridge estimates
for each provider--task-bucket cell. It normalizes receipt cost by
\(c_{\max}=\$0.01\) and scores affordable providers by
\(\mathrm{UCB}_r-\mu\,\mathrm{LCB}_c\). Before each selection, \(\mu\) is
increased by \(0.01\) when the remaining-budget fraction is below the
remaining-horizon fraction and decreased by \(0.01\) otherwise, subject to
the range \([0,100]\). Only the chosen cell's utility and normalized receipt
cost are updated; no temporal forgetting is applied.

\paragraph{Replay True Oracle.}
We implement the replay True Oracle with an independent wallet trajectory. For
the same replay task and cached-response version, it inspects the realized
outcome of every provider affordable under that wallet, evaluates each using
the oracle wallet's current payment pressure, and selects the provider with
the highest payment-aware reward. Only the selected receipt is charged. This
counterfactual access makes the oracle unattainable and suitable only as an
upper bound.

\subsection{Comparator Configurations}
\label{app:comparator-configurations}

Table~\ref{tab:comparator-configurations} reports the fixed configurations of
comparators with explicit policy parameters, including the appendix-only
SW-TS diagnostic of Appendix~\ref{app:sw-ts}.

\begin{table*}[!t]
\centering
\small
\setlength{\tabcolsep}{5pt}
\begin{tabularx}{\textwidth}{
    @{}
    >{\raggedright\arraybackslash}p{0.20\textwidth}
    >{\raggedright\arraybackslash}X
    @{}
}
\toprule
Comparator & Configuration \\
\midrule

BudgetRule
& High and low remaining-budget thresholds are \(0.50\) and \(0.20\),
respectively. \\

Contextual DS-TS
& \(K_b=4\); Q-prior mean, prior variance, and likelihood variance
\((0.5,1.0,0.09)\); wall-clock discount \(\gamma=0.999\). \\

Contextual BTS
& \(K_b=4\); Q settings \((0.5,1.0,0.09)\); C-prior mean \(\bar c_a\);
C-prior and likelihood variances \((10^{-4},10^{-6})\); \(\gamma=1\);
sampled-cost floor \(10^{-6}\) USDC. \\

PM-Greedy
& \(K_b=4\); per-pull window \(W=500\); utility threshold
\(\theta_{\mathrm{PM}}=0.7\); optimistic initial estimate
\(\hat u_{\mathrm{init}}=1\). \\

LinCBwK-Adapt
& \(K_b=4\); reward and normalized-cost prior means \(0.5\);
corresponding prior and likelihood variances \(0.09\); \(\gamma=1\);
\(c_{\max}=\$0.01\); \(\beta=1.0\); dual step \(\eta=0.01\) and range
\([0,100]\); cost-LCB floor \(10^{-6}\). \\

SW-TS diagnostic
& Non-contextual; wall-clock window \(W=1000\); reward-prior mean and
variance \((0.5,1.0)\); likelihood variance \(0.09\). \\

\bottomrule
\end{tabularx}
\caption{Comparator configurations used in the reported experiments.
Parameters are fixed across scenarios and random seeds unless explicitly
varied in Appendix~\ref{appendix:sensitivity}.}
\label{tab:comparator-configurations}
\end{table*}

\paragraph{Alignment and numerical guards.}
Contextual DS-TS shares \padct{}'s Q-prior, likelihood variance, and discount.
Contextual BTS shares \padct{}'s Q/C prior settings and likelihood variances
but removes discounting, reducing differences due solely to prior
calibration. LinCBwK-Adapt instead uses narrower ridge-style prior variances
and does not seed its normalized-cost prior from catalog prices. The
\(10^{-6}\) cost floors in Contextual BTS and LinCBwK-Adapt prevent sampled
or lower-confidence costs near zero from destabilizing their selection
scores.

\section{x402 Integration Witness}
\label{appendix:x402-witness}

The released artifact includes a local x402 integration witness for the
\texttt{PaymentExecutor} boundary
\ifdefined\technicalreportonly
used by 402Pilot.
\else
described in \S\ref{sec:layer} of the main paper.
\fi
The witness demonstrates that a selected-provider purchase can execute as a
real HTTP~402 quote--pay--receipt loop. It is separate from
402Pilot-Bench and contributes no reported measurements.
The local setup consists of an Anvil fork containing USDC contract state, a
Python x402 facilitator for verifying and settling exact EVM payments, a
FastAPI resource server exposing x402-protected provider endpoints, and an
\texttt{X402PaymentExecutor} adapter that maps the HTTP~402 exchange to a
402Pilot \texttt{Outcome}.

\paragraph{Runtime environment.}
The reported benchmark experiments are deterministic, CPU-only frozen-replay
runs conducted on a MacBook Pro with an Apple M2~Max and 32\,GB unified
memory. The released artifact pins the software dependencies and exact
versions. The integration witness additionally requires a local Anvil
instance, the Python x402 facilitator, and the FastAPI resource server.
Running \texttt{bash scripts/witness\_smoke.sh} exercises both the unpaid
HTTP~402 response and a paid request, and verifies a valid payment receipt
together with the corresponding decrease in the buyer's local Anvil USDC
balance.

\FloatBarrier
\putbib[references]
\end{bibunit}

\end{document}